\documentclass[11pt]{article}
\date{}

\usepackage[utf8]{inputenc} 
\usepackage[T1]{fontenc}    
\usepackage[colorlinks=true,linkcolor=blue,allcolors=blue]{hyperref}      
\usepackage{url}           
\usepackage{booktabs}     
\usepackage{amsfonts}   
\usepackage{nicefrac}       
\usepackage{microtype,nicefrac}     
\usepackage{xspace}
\usepackage{natbib}
\usepackage{wrapfig}
\usepackage{comment}
\usepackage{amsmath}
\usepackage{amsthm}
\usepackage{graphicx}
\usepackage{rotating}
\usepackage{caption}
\usepackage{subcaption}
\usepackage{amssymb}
\usepackage{autobreak}
\usepackage{mathtools}
\usepackage[dvipsnames]{xcolor}
\usepackage{bbm}
\usepackage[ruled,vlined, linesnumbered]{algorithm2e}
\usepackage{algorithmic}
\usepackage[parfill]{parskip}
\usepackage{authblk}
\usepackage{nicefrac}      
\usepackage{microtype}     
\usepackage[dvipsnames]{xcolor}         
\usepackage{float}
\usepackage{amssymb}
\usepackage{mathtools}

\makeatletter

\newcommand{\Var}{\mathbb{V}}
\newcommand{\E}{\mathbb{E}}
\newcommand{\R}{\mathbb{R}}

\newcommand{\A}{\mathcal{A}}

\newcommand{\M}{\mathcal{M}}
\newcommand{\D}{\mathcal{D}}
\newcommand{\Ss}{\mathcal{S}}
\newcommand{\F}{\mathcal{F}}

\newcommand{\cvar}{\text{CVaR}}
\newcommand{\cvara}{\text{CVaR}_{\alpha}}
\newcommand{\ccara}{\text{CCaR}_{\alpha}}

\newcommand{\wt}{\widetilde}
\newcommand{\wh}{\widehat}
\newcommand{\ol}{\overline}

\newcommand{\one}{\mathbbm{1}}

\newcommand{\Prob}{\mathbb{P}}
\newcommand{\dkl}{\text{D}_{\text{KL}}}

\newcommand{\llbrace}{\left\lbrace}
\newcommand{\rrbrace}{\right\rbrace}
\newcommand{\lb}{\left[} 
\newcommand{\rb}{\right]}
\newcommand{\lv}{\left\vert} 
\newcommand{\rv}{\right\vert}
\newcommand{\lV}{\left\Vert} 
\newcommand{\rV}{\right\Vert}
\newcommand{\lp}{\left(} 
\newcommand{\rp}{\right)}

\newcommand{\CDF}{\textsc{{CDF}}\xspace}
\newcommand{\IS}{\textsc{IS}\xspace}
\newcommand{\FIS}{\textsc{F-IS}\xspace}
\newcommand{\SIS}{\textsc{S-IS}\xspace}
\newcommand{\CIS}{\textsc{C-IS}\xspace}
\newcommand{\WIS}{\textsc{WIS}\xspace}
\newcommand{\DR}{\textsc{DR}\xspace}
\newcommand{\DM}{\textsc{DM}\xspace}
\newcommand{\WDR}{\textsc{WDR}\xspace}
\newcommand{\ISclip}{\textsc{IS-clip}\xspace}
\newcommand{\MDR}{\textsc{M-DR}\xspace}
\newcommand{\MDP}{\textsc{\small{MDP}}\xspace}

\newtheorem{theorem}{Theorem}[section]
\newtheorem{lemma}{Lemma}[section]

\newtheorem{corollary}{Corollary}[section]
\newtheorem{proposition}{Proposition}[section]
\newtheorem{definition}{Definition}[section]
\newtheorem{assumption}[theorem]{Assumption}

\title{Off-Policy Risk Assessment in Markov Decision Processes}

\author[1]{Audrey Huang}
\author[2]{Liu Leqi}
\author[2]{Zachary C. Lipton}
\author[3]{Kamyar Azizzadenesheli}
\affil[ ]{\texttt{\href{mailto:audreyh5@illinois.edu}{\textcolor{black}{audreyh5@illinois.edu}},\href{mailto:leqil@cs.cmu.com}{\textcolor{black}{leqil@cs.cmu.edu}},\href{mailto:zlipton@cmu.edu}{\textcolor{black}{zlipton@cmu.edu}},\href{mailto:kamyar@purdue.edu}{\textcolor{black}{kamyar@purdue.edu}}}}
\affil[1]{Department of Computer Science, University of Illinois Urbana-Champaign}
\affil[2]{Machine Learning Department, Carnegie Mellon University}
\affil[3]{Department of Computer Science, Purdue University}

\newcommand\blfootnote[1]{%
  \begingroup
  \renewcommand\thefootnote{}\footnote{#1}%
  \addtocounter{footnote}{-1}%
  \endgroup
}

\begin{document}
\maketitle
\blfootnote{}
\blfootnote{Published at the 25th International Conference on Artificial Intelligence and Statistics (AISTATS) 2022.}

\begin{abstract}
  Addressing such diverse ends as safety
alignment with human preferences,
and the efficiency of learning,
a growing line of reinforcement learning research 
focuses on risk functionals
that depend on the entire distribution of returns.
Recent work on \emph{off-policy risk assessment} (OPRA)
for contextual bandits
introduced consistent estimators 
for the target policy's \CDF of returns 
along with finite sample guarantees 
that extend to (and hold simultaneously over) 
all risk.
In this paper, we lift OPRA to
Markov decision processes (\MDP{}s),
where importance sampling (IS) \CDF estimators
suffer high variance on longer trajectories
due to small effective sample size.
To mitigate these problems,  
we incorporate model-based estimation
to develop the first 
doubly robust (DR) estimator 
for the \CDF of returns in \MDP{}s. 
This estimator enjoys 
significantly less variance and, 
when the model is well specified,
achieves the Cramer-Rao variance lower bound.
Moreover, for many risk functionals,
the downstream estimates enjoy 
both lower bias and lower variance. 
Additionally, we derive the first minimax lower bounds 
for off-policy CDF and risk estimation, 
which match our error bounds up to a constant factor.  
Finally, we demonstrate the 
precision
of our \DR \CDF estimates experimentally
on several different environments.

\end{abstract}

\section{Introduction} 
In off-policy evaluation (OPE),
we aim to estimate the risk
associated with a \emph{target} policy
given only offline datasets collected
from a deployed \emph{behavior} policy.
Broadly, a risk functional maps 
from \emph{the distribution of returns} 
to a real value.
While OPE research has historically 
focused on estimating the expected return,
a wave of recent works in reinforcement learning (RL)
focus on other risk functionals,
motivated by such diverse desiderata
as risk aversion (and seeking)
and alignment with human preferences.
Examples of such alternative risk functionals 
include  the expectation, variance, 
value at risk, conditional value at risk,
and broader families such 
as distortion risk functionals,
coherent risk functionals, 
and Lipschitz risk functionals,
Notably, the Lipschitz risks,
introduced by \citet{huang2021off}
subsume most common risk functionals.

In a new line of work on \emph{off-policy risk assessment} (OPRA),
researchers have developed methods 
for simultaneously estimating arbitrarily many risk functionals,
providing guarantees
that hold simultaneously over 
an entire collection of risks
\citep{huang2021off, chandak2021universal}.
These results are achieved 
through a two-step process:
(i) estimate the cumulative distribution function (\CDF) 
of returns under the target policy;
and (ii) calculate the risk functionals of interest
on the estimated \CDF. 
This approach has two advantages:
(a) guarantees on the \CDF estimation error
yield corresponding guarantees on the risk estimates;
and (b) because the risk estimates 
all share the same underlying \CDF,
the corresponding guarantees on any number of
downstream risk estimates hold simultaneously
without loss in statistical power. 

To our knowledge, for \MDP{}s, 
importance sampling (IS) 
is the only established method 
for estimating off-policy \CDF{}s \citep{chandak2021universal}. 
However, IS estimators are known to suffer 
from high variance,
render them impractical over long trajectories.
This variance can be exponential 
in the horizon \citep{liu2018breaking},
as weights vanish for most trajectories
and explode for others,
yielding tiny effective sample sizes.
Consider, for example, that practitioners 
often evaluate a deterministic target policy
given a dataset 
collected by a stochastic behavioral policy. 
In such cases, the \IS weight is zero for all trajectories
where the target policy places zero probability 
on \emph{any} action that the behavioral policy takes,
even if the policies agree on most actions.
This renders most trajectories unusable,
shrinking the effective size 
of an offline dataset \citep{sussexstitched}.

In off-policy evaluation addressing 
the standard \emph{expected return} risk functional,
doubly robust estimators and their variants 
have been shown to mitigate these problems 
\citep{dudik2011doubly, dudik2014doubly, jiang2016doubly, thomas2016data, wang2017optimal, su2020doubly}. 
In the contextual bandit setting,
the paper introducing OPRA also proposed and analyzed 
a doubly robust estimator for the \CDF \citep{huang2021off}, 
but model-based and doubly robust risk estimation in \MDP{}s,
where it may be most impactful, 
remains an open problem. 
One primary challenge is that, 
in the expected return setting, 
extending the model-based and doubly robust estimators 
from contextual bandits to \MDP{}s \citep{jiang2016doubly} 
relied on stepwise or recursive formulations 
of the importance sampling estimator, 
which lack clear analogs in \CDF estimation. 

In this paper, we define and analyze 
the first doubly robust \CDF estimator 
for off-policy risk assessment in \MDP{}s. 
Specifically, we contribute the following:
\begin{enumerate}
    \item A recursive formulation of the \CDF of returns in \MDP{}s (\S~\ref{sec:recursive}), 
            which enables us to derive importance sampling (\S~\ref{sec:is_cdf}), 
            and model-based \CDF estimators (\S~\ref{sec:dm_cdf}). 
    \item Rate-matching upper
            and minimax lower bounds (\S~\ref{sec:is_error}) 
            for the IS \CDF estimator. 
    \item Variance-reduced \IS estimators leveraging unique properties 
        of the return distribution (\S~\ref{sec:s_vs_f_is}). 
    \item The first doubly robust off-policy estimator 
            for the \CDF of returns (\S~\ref{sec:dr_cdf}),
            and analysis demonstrating its potential for significant variance reduction.
    \item The first Cramer-Rao lower bound for off-policy 
            \CDF estimation and analysis demonstrating 
            that the doubly robust estimator 
            achieves this lower bound in certain cases
            (\S~\ref{sec:variance_lower_bound}). 
    \item A demonstration that the risk estimates 
            inherit the reduced variance and error in \CDF estimation (\S~\ref{sec:risk}). 
    \item Experimental evidence that the doubly robust \CDF and risk estimates 
            exhibit significantly lower error and variance 
            than importance sampling estimators on benchmark problems, 
            with applications to safe policy evaluation 
            and improvement (\S~\ref{sec:experiments}). 
\end{enumerate}

\section{Related Work} 

Several estimators have been studied
for off-policy estimation of the expected return. 
A seminal work by \citet{dudik2011doubly} 
proposed the doubly robust estimator 
for the expected return in contextual bandits, 
which was extended to the MDP setting by \cite{jiang2016doubly}.
The doubly robust estimator leverages 
both importance sampling weights and models 
to reduce variance while retaining unbiasedness. 
Researchers have proposed several variants 
of the doubly robust estimator,
aiming to improve its performance,
with \cite{thomas2016data, wang2017optimal} optimizing the balance 
between importance sampling and model-based estimators
and \cite{su2020doubly} optimizing the weights used in the estimator. 
A Cramer-Rao variance lower bound~\citep{jiang2016doubly} 
and minimax lower bounds~\citep{li2015toward, wang2017optimal,ma2021minimax} 
for off-policy estimation of the mean have been derived.

While the expected return has been the focus 
of most previous off-policy evaluation research, 
other risk functionals are increasingly of interest. 
In the RL literature, the variance 
\citep{sani2013risk, tamar2016learning}, 
variance-constrained mean \citep{mannor2011mean},
and conditional value-at-risk 
\citep{rockafellar2000optimization, keramati2020being, huang2021convergence} 
are often employed. 
Other risk functionals,
including exponential utility \citep{denardo2007risk}, 
distortion risk functionals \citep{dabney2018implicit},
and cumulative prospect weighting
\citep{gopalan2017weighted, prashanth2016cumulative, prashanth2020concentration}
have also been investigated. 

Several recent works have aimed to estimate risk functionals 
from off-policy datasets. 
Of these, \citet{chandak2021highconfidence} estimates the variance,
while more recent works \citep{huang2021off, chandak2021universal} 
tackle the estimation of more general risks
and are the closest works of comparison. 
Both~\citet{huang2021off, chandak2021universal} 
take a two-step approach 
of first estimating the off-policy CDF of returns;
and then estimating their risks via a plug-in approach. 
\citet{chandak2021universal} proposed an IS \CDF estimator
in both stationary and nonstationary \MDP{}s 
and derives confidence intervals 
for a number of risk estimates on a per-risk basis.
\cite{huang2021off} proposed both IS and DR estimators 
for the \CDF of returns in contextual bandits,
establishing the first finite sample 
guarantees for off-policy \CDF estimation
and showing that the estimated CDFs converge uniformly,
achieving $O(1/\sqrt{n})$ rates.
Additionally, they introduced the broad class of Lipschitz risk functionals,
where the errors in risk estimation are upper-bounded 
by the error in CDF estimation multiplied 
by each functional's Lipschitz constant. 

While \citet{chandak2021universal} proposes \emph{only} IS CDF estimates for MDPs, \citet{huang2021off} proposes variance-reduced DR CDF estimates 
in \emph{only} contextual bandits. 
Thus \emph{variance-reduced risk estimation in MDPs}, 
where high variance from IS weights is most problematic~\citep{liu2018breaking}, 
is an open problem. 
Towards this end, we (1) develop provably variance-reduced CDF and risk estimators in MDPs, and, (2) demonstrate their optimality 
by deriving the first matching upper and lower bounds.

\section{Preliminaries}\label{sec:preliminaries}

\subsection{Problem Setting}
An MDP is defined by $\M = (\Ss, \A, P, R, \gamma, \mu)$, 
where $\Ss$ is the state space, $\A$ is the action space, 
$P : \Ss \times \A \times \Ss \rightarrow \R$ is the transition function, 
$R : \Ss \times \A \rightarrow \R$ is the reward function, 
$\gamma \in (0, 1)$ is the discount factor, 
and $\mu$ is the fixed starting state distribution.
We assume that rewards are bounded on the support $[0, D]$. 
A stationary policy $\pi : \Ss \rightarrow \Delta(\A)$
maps a state to a probability distribution over actions. 

In the off-policy evaluation problem, 
we are concerned with estimating 
risk functionals of a policy $\pi$ 
given a dataset of trajectories $\D = \{\tau^i\}_{i=1}^n$ 
collected by the behavioral policy $\beta$ interacting with MDP $\M$.
Each $\tau = (S_1, A_1, R_1, \ldots, S_{H}, A_{H}, R_{H}, S_{H+1})$ is an $H$-step trajectory, with $S_1 \sim \mu$, $R_h \sim R(S_h, A_h)$, 
and $S_{h+1} \sim P(\cdot|S_h,A_h)$. 
In general, we will refer to random variables with capital letters, 
and for convenience, we denote $\tau_j := (S_1, A_1, R_1, \ldots S_j, A_j, R_j)$,
or a sub-trajectory of $\tau$ from horizon $1$ to $j$. 

Further, let $w(A_h, S_h) := \frac{\pi(A_h|S_h)}{\beta(A_h|S_h)}$ 
denote the importance weight, and let
$w_{max} := \max_{h, A_h, S_h}w(A_h, S_h)$
be the maximum importance weight over 
all horizons, states, and actions. 
For convenience, by 
$w_{j} := \prod_{h=1}^j w(A_h, S_h)$, 
we denote the importance weight 
of subtrajectory $\tau_j$.
Thus, $w_H$ is the importance weight of the trajectory $\tau$. 

We denote by $\Prob$ the distribution of trajectories
induced by $\pi$ on $\mathcal{M}$, 
and by $\Prob_\beta$ the distribution induced by $\beta$. 
The random variable of returns 
of an $H$-step trajectory 
induced by $\pi$ in $\M$ 
is given by 
$
 Z^{\pi} = \lp \sum_{h=1}^{H} \gamma^{h-1} R_h \Big\vert \pi \rp
$. 
We also write the random variable of returns 
starting from a state $s_h$ at horizon $h$, 
summed until the end of horizon $H$,
to be 
$Z^{\pi}_{s_h} = \lp \sum_{k=h}^{H} \gamma^{k-h} R_k \Big\vert \pi, s_h \rp$. 
The random variable of returns conditioned 
on a state and action at horizon $h$,
$Z^{\pi}_{s_h, a_h}$, is defined similarly. 

The cumulative distribution function (CDF) of $Z^\pi$ is
$F^{\pi}(t) = \E_\Prob \lb \one\llbrace Z^{\pi} \leq t \rrbrace \rb$, 
where $\one\llbrace \cdot \rrbrace$ is the indicator function. 
Similarly, the \CDF of conditional returns $Z_{s_h}^\pi$
is defined as 
{$F^\pi_{s_h}(t) = \E_\Prob \lb \one\llbrace Z^\pi_{s_h} \leq t \rrbrace | s_h\rb$},
and the state-action dependent distribution
$F_{s_h, a_h}(t)$ is defined similarly. 
We also define the \emph{complementary cumulative distribution function} (CCDF) 
$S^{\pi}(t) := 1 - F^{\pi}(t) = \E_\Prob \lb \one\llbrace Z^{\pi} > t \rrbrace \rb$, 
and $S_{s_h}^\pi(t) := 1 - F_{s_h}^\pi(t)$. 

For short, we write $\E_h^\beta[\cdot] := \E_{\Prob_\beta}[\cdot|\tau_{h-1}]$ 
to be the expectation under behavioral policy $\beta$, 
conditioned on a trajectory up until time $h$.
Similarly, for $\pi$, 
we write the conditional expectation 
$\E_h[\cdot] := \E_{\Prob}[\cdot|\tau_{h-1}]$.
We define the conditional variances $\Var_h^\beta, \Var_h$ similarly, 
e.g., $\Var_h^\beta = \Var_{\Prob_\beta}[\cdot|\tau_{h-1}]$.

\subsection{CDF Bellman Equation}\label{sec:recursive}
Our off-policy CDF estimators 
rely on a novel \emph{CDF Bellman equation}
that provides a 
recursive formulation of the \CDF of returns $F$.
In the off-policy setting, 
for state $s_h$ at horizon $h$, 
we have:
\begin{align}
    F_{s_h}(t) 
    &:= \E_{\Prob_\beta}\lb w(A_h, s_h)  F_{S_{h+1}}\lp \frac{t - R_h}{\gamma} \rp \Big| s_h\rb \label{eq:recursive_cdf}
\end{align}
where $F_{s_{H+1}} = \one\{ 0 \leq t \}$ 
for all $s_{H+1}$ at the end of the horizon. 
Note that the CDF of returns under $\pi$ can also be written as $F^{\pi}(t) = \E_{s_0 \sim \mu} [F_{s_0}^{\pi}(t)]$, i.e. the expectation under the starting state distribution $s_0 \sim \mu$ of the CDFs $F^{\pi}_{s_0}$.

This recursion was first introduced in \cite{sobel1982variance} 
for deterministic rewards in the value iteration setting, 
which we have extended to the more general stochastic setting. 
Using the definition of equivalence in distribution ($\stackrel{D}{=}$),
it can be seen that the CDF Bellman operator 
is equivalent to the distributional Bellman operator~\citep{bellemare2017distributional, dabney2018implicit}, 
$Z_{s_h} \stackrel{D}{=} R(s_h, A_h) + \gamma Z_{S_{h+1}}$.

\subsection{Risk Functionals}
Let $Z\in \mathcal{L}_\infty(\Omega, \F_Z,\Prob_Z)$ 
denote a real-valued random variable 
that admits a \CDF 
$F_Z\in \mathcal{L}_\infty(\R, \mathcal{B}(\R))$. 
A \emph{risk functional} 
$\rho: \mathcal{L}_\infty(\Omega, \F_Z,\Prob_Z) \rightarrow \R$ 
is said to be law-invariant if $\rho(Z)$ depends 
only on the distribution of $Z$ \citep{kusuoka2001law},
i.e., if for any pair of random variables $Z$ and $Z'$,
$F_Z = F_{Z'}\;\Longrightarrow\;\rho(Z)=\rho(Z')$. 
When clear from the context,
we write $\rho(F_Z)$ in place of $\rho(Z)$.
In this paper, we restrict our focus 
to law invariant risk functionals, 
as it may not be practical 
to estimate risk functionals 
that are not law invariant
from data \citep{balbas2009properties}.

Examples of popularly studied law-invariant risk functionals
include the expected value, variance, mean-variance, 
coherent risks (including conditional value-at-risk (CVaR) and proportional hazard), 
distortion risks, 
and
cumulative prospect theory risks,
which we discuss more extensively in Appendix~\ref{appendix:preliminaries} due to space constraints. 
We give particular attention to distortion risk functionals due to their widespread usage and general form. Distortion risks apply a distortion function $g: [0,1] \to [0,1]$ to the CDF $F_Z$ and have the formula
$$\rho(Z) = \int_0^\infty g(1 - F_Z(t))dt = \int_0^\infty g(S(t))dt.$$ 
Notably,  expected value is a special case of distortion risk with $g(x) = x$, while CVaR at level $\alpha$ ($\cvara$) has a distortion function $g(x) = \min\{ \frac{x}{1-\alpha}, 1\}$. A large number of the aforementioned risk functionals satisfy a notion of smoothness in the CDF, which~\cite{huang2021off} define as the class of \emph{Lipschitz risk functionals}: 
\begin{definition}[Lipschitz Risk Functionals]\label{def:lipschitz}
    A law invariant risk functional $\rho$ is $L$-Lipschitz 
    if there exists $L \in [0, \infty)$ such that for any pair of \CDF{}s $F_Z$ and $F_{Z'}$, 
    it satisfies
    \begin{equation}
        |\rho(F_Z)-\rho(F_{Z'})|\leq L \|F_Z-F_{Z'}\|_\infty. 
    \end{equation}
\end{definition}

Lipschitz risk functionals include distortion risk functionals with Lipschitz distortion functions $g$. As special cases, expected value is $D$-Lipschitz and CVaR at level $\alpha$ is $D/(1-\alpha)$-Lipschitz. Additionally, variance is ($3D^2$-)Lipschitz and so is mean-variance~\citep{huang2021off}. 

As we analyze the theoretical properties
of CDF and risk estimators later in our paper, 
we show that estimates of Lipschitz risk functionals 
retain many desirable properties, including 
characterizable finite sample error and consistency.

\subsection{Off-Policy Risk Assessment}
In the remainder of this paper, we will take the general approach 
of Algorithm~\ref{algo:opra} (OPRA) for off-policy risk assessment. 
OPRA takes as input a set of risk functionals of interest, 
and outputs estimates of their values under the target policy $\pi$. 
It first estimates the CDF of returns under $\pi$ from off-policy data,
then generates plug-in risk estimates on the CDF. 
In Sections~\ref{sec:is_cdf},~\ref{sec:dm_cdf} and~\ref{sec:dr_cdf}, 
we define and analyze CDF estimators. 
In Section~\ref{sec:risk}, we turn to plug-in risk estimation.  
\vspace*{-0.5em}
\begin{algorithm}
    \SetAlgoLined 
    \KwIn{Dataset $\mathcal{D}$, policy $\pi$, probability $\delta$, 
    risk functionals $\left\{ \rho_p \right\}_{p=1}^P$ with Lipschitz constants $\left\{ L_p \right\}_{p=1}^P$. }
    Estimate the CDF $\wh{F}$ with error $\epsilon$\; 
    For $p = 1...P$, estimate risk $\wh{\rho}_p = \rho_p(\wh{F})$\; 
    \KwOut{Estimates with errors $\lbrace \wh{\rho}_p \pm L_p\epsilon\rbrace_{p=1}^P$.  }
    \caption{Off-Policy Risk Assessment (OPRA)}
    \label{algo:opra} 
\end{algorithm}
\vspace*{-1em}

\section{Importance Sampling CDF Estimation}\label{sec:is_cdf}
Off-policy evaluation faces the unique challenge 
that only rewards for actions taken 
by the behavioral policy are observed. 
Importance sampling (IS), which reweights samples 
according to the ratio between target and behavioral policy probabilities, 
can be used to 
account for
this distribution shift. 
While the IS estimator for off-policy estimation was previously introduced in~\cite{huang2021off, chandak2021universal}, 
finite sample convergence in MDPs 
and lower bounds for off-policy 
CDF estimation (in general)
remained open questions. 
In this section, we derive the first minimax 
lower bound for off-policy CDF estimation,
and show that it matches the IS estimator upper bound up to a constant factor.  
Further, we leverage unique properties of the return distribution 
to develop new variance-reduced IS estimators. 

\subsection{Importance Sampling Estimator}\label{sec:is_estimator}
The 
IS
\CDF estimator forms an empirical estimate of the recursion in~\eqref{eq:recursive_cdf}: 
\begin{align}
    &\wh{F}_\IS(t) = \frac{1}{n} \sum_{i=1}^n \wh{F}_{s_1^i}(t), \label{eq:recursive_is_estimator} \\
    &\quad\text{where}\quad \wh{F}_{s_h^i}(t) = w(a_h^i, s_h^i)\wh{F}_{s_{h+1}^i}\lp \frac{t - r_h^i}{\gamma} \rp \nonumber  
\end{align}
At each horizon $h$, the recursive IS estimator applies the importance weight to the estimated CDF at the next state $s_{h+1}$ convolved with the reward $r_h$ and discount factor $\gamma$. 
Unrolling the recursion recovers the more intuitive form of the IS estimator, also introduced in~\cite{chandak2021universal},
that reweights each sample by the importance weight of the entire trajectory:
\begin{equation}\label{eq:is_estimator}
    \wh{F}_\IS(t) = \frac{1}{n} \sum_{i=1}^n w_H^i \one\llbrace\sum_{h=1}^H \gamma^{h-1} r_h^i \leq t \rrbrace. 
\end{equation}

\begin{lemma}\label{lem:is_bias_var}
    The \IS \CDF estimator~\eqref{eq:recursive_is_estimator} is unbiased 
    and its variance is recursively given by 
    {
    \begin{align}
        \Var_h^\beta \lb\wh{F}_{S_h}(t)\rb =& \E_h^\beta \lb \Var_h^\beta \lb w(A_h, S_h)F_{S_h, A_h}(t) \Big\vert S_h \rb\rb \nonumber \\
        &+ \E_{h}^\beta\lb w(A_h, S_h)^2 \Var_{h+1}\lb\wh{F}_{S_{h+1}}\lp \frac{t - R_h}{\gamma}\rp \Big\vert S_h, A_h \rb \rb \nonumber \\
        &+ \Var_h\lb F_{S_h}(t)\rb. \label{eq:is_variance}
    \end{align}
    }%
\end{lemma}

The first term expresses variance from importance-weighted actions sampled as $A_h \sim \beta(\cdot|S_h)$,
while the third term expresses variance from transitions. 
The second term encompasses variance 
from random rewards and transitions, and recurses to the next time step.
At horizon $h$, each term in the variance
is multiplied by the product of importance weights
from the trajectory thus far, $w_{h-1}$.

\subsection{Finite Sample Error Bound}\label{sec:is_error}
While the \IS estimator is known to be consistent, 
obtaining finite sample error bounds in MDPs
has remained an important open problem. 
Such bounds characterize the convergence rate 
of the estimate towards the true $F$, 
and provide confidence intervals around estimates
in the finite sample regimes of practical scenarios.

First, however, it is important to note that, 
given finite samples,
the IS estimator may not be a valid \CDF 
due to the use of importance weighting, 
e.g., it can be greater than 1. 
To address this issue, we can use 
a weighted importance sampling (\WIS) estimator instead:  
$
    \wh{F}_\WIS(t) = \frac{n} {\sum_{j=1}^n w_H^j}\wh{F}_\IS(t)
$,
which normalizes the estimate 
using the sum of importance weights 
instead of the sample size, 
guaranteeing that $\wh{F}_\WIS \in [0, 1]$.   
As \cite{chandak2021universal} demonstrates,
the \WIS estimator is biased
but uniformly consistent.

Alternatively, the \IS estimator 
can be clipped to the unit range, a strategy also employed by \cite{huang2021off} 
for off-policy \CDF estimation in contextual bandits, shown below: 
$$
\wh{F}_{\ISclip}(t) = \min\llbrace \wh{F}_\IS(t), 1\rrbrace,
$$
The \ISclip estimator, too, is biased but uniformly consistent, 
and its error is uniformly bounded in the following lemma:
\begin{lemma}\label{lem:is_error}
    With probability at least $1-\delta$, for universal constants $c_1, c_2 >0$
    \begin{align}
        \Vert \wh{F}_\ISclip - F \Vert_\infty \leq c_1\sqrt{\frac{\E_{\Prob_\beta}[w_H^2]}{n}} + c_2 \frac{w_{max}}{n}. 
    \end{align}
\end{lemma}

The \IS estimator converges at a rate of $O(\E_{\Prob_\beta}[w_H^2]/\sqrt{n})$, subsuming previous results on convergence in contextual bandits~\citep{huang2021off}. For readability throughout the paper, we leave constants implicit and give their exact form in the proof (Appendix~\ref{appendix:is}). 

How close to optimal is this convergence rate? For large enough sample size $n$, the \IS estimator is in fact minimax rate-optimal, as the following lower bound demonstrates:
\begin{theorem}\label{thm:cdf_lower_bound}
    For a universal constant $c$, when $n \geq c \frac{w_{max}^2}{\E_{\Prob_\beta}[w_H^2]}$, 
    we have 
    \begin{align}
        \inf_{\wh F } \sup_{F \in \F} \E\lb \Vert \wh{F} - F^\pi \Vert_\infty \rb \geq c\sqrt{\frac{\E_{\Prob_\beta}[w_H^2]}{n}}
    \end{align}
\end{theorem}

\subsection{Variance-Reduced IS Estimation}\label{sec:s_f_estimator}
Unlike the \IS estimator of the expected return~\citep{dudik2011doubly,jiang2016doubly},
the \IS estimator of the CDF 
has the unique property 
that its variance differs across 
its support as a function of $t$. 
Examining the variance
for $\wh{F}_\IS$~\eqref{eq:is_variance} at a fixed $t$, 
the first term is small if $F_{S_h, A_h}(t)$ is close to 0, 
which can occur at small $t$. 
However if $F_{S_h, A_h}(t)$ is close to its maximal value of $1$, 
which occurs when $t$ is close to the maximum return $D$, 
the first term is close to its largest possible value. 

\subsubsection{Estimating the Complementary Cumulative Distribution Function (CCDF)}\label{sec:s_is}
Rather than estimating the CDF $F$, 
another possibility is to estimate the CCDF $S = 1 - F$, 
which is directly used in the expression 
for many risk expressions (e.g. distortion risk functionals, see \S~\ref{sec:preliminaries}). 
IS estimators for $S$ 
take a similar form as those for $F$, 
but with the indicator function inequality
swapped to $\one\{ \; \cdot \;>t\}$.
For example, 
\begin{equation}\label{eq:s_estimator}
    \wh{S}_\IS(t) = \frac{1}{n} \sum_{i=1}^n w_H^i \one\llbrace\sum_{h=1}^H \gamma^{h-1} r_h^i > t \rrbrace, 
\end{equation}
Going forward, we refer to the estimator in~\eqref{eq:is_estimator} as $\wh{F}_\FIS$, and $\wh{F}_\SIS := 1 - \wh{S}_\IS$ from~\eqref{eq:s_estimator} to avoid confusion.
Perhaps non-intuitively, $\wh{F}_\FIS$ and $\wh{F}_\SIS$ have different theoretical properties, namely its variance: 
{
\begin{align*}
    \Var_h^\beta \lb \wh{F}_{\SIS, \;S_h}(t)\rb = \;&\Var_h^\beta \lb1 - \wh{S}_{S_h}(t)\rb = \E_h^\beta \lb \Var_h^\beta \lb w(A_h, S_h)(1 - F_{S_h, A_h}(t)) \Big\vert S_h \rb\rb \\
    &\quad+ \E_{h}^\beta\lb w(A_h, S_h)^2 \Var_{h+1}\lb\wh{S}_{S_{h+1}}\lp \frac{t - R_h}{\gamma}\rp \Big\vert S_h, A_h \rb \rb \\
    &\quad+ \Var_h\lb S_{S_h}(t)\rb
\end{align*}
}
\normalsize

The first term of the variance expression for $\wh{F}_\SIS$ 
takes the exact opposite trend as that of $\wh{F}_\FIS$: 
it is low where $F_{S_h, A_h}(t)$ is high, and vice versa. 
For example, if $\wh{F}_\FIS$ has low variance in estimating the lower tails but high variance in the upper tails, 
$\wh{F}_\SIS$ will have low variance in the upper tails 
but high variance in the lower tails.
Lower variance leads to lower error of the CDF estimate on certain portions of the distribution. 
To formalize this argument, we present the following bounds 
on the error of the CDF estimate 
in terms of its variance at each $t$. 
As we later show, this will also have important consequences for downstream estimation of different risks, which up-weight different parts of the distribution. 

\begin{lemma}\label{lem:s_f_bernstein}
    If $\Var[\wh{F}_\IS]$ and $\Var[\wh{S}_\IS]$ have bounded variation (see Assumption~\ref{assum:bounded_var}), 
    there exists universal constants $c_1, c_2$ 
    such that with probability at least $1-\delta$, $\forall t$, 
    \begin{align*}
        |F(t) - \wh{F}_\FIS(t)| &\leq \frac{c_1\log(\sqrt{n}/\delta)}{n} + \sqrt{\frac{c_2\Var[\wh{F}_\IS(t)]\log(\sqrt{n}/\delta)}{n}}, \\
        |F(t) - \wh{F}_\SIS(t)| &\leq \frac{c_1\log(\sqrt{n}/\delta)}{n} + \sqrt{\frac{c_2\Var[\wh{S}_\IS(t)]\log(\sqrt{n}/\delta)}{n}} 
    \end{align*}%
    \normalsize
\end{lemma}

This difference has important implications for downstream risk estimation. 
The choice of whether to estimate $S$ or $F$ 
can reduce the variance of the plug-in risk estimator,  
especially for risk functionals such as CVaR 
that heavily up-weight one tail of the distribution. 
In general, under reward distributions, 
$\wh{F}_\SIS$ will lead to lower variance estimators 
for risk-averse risk functionals 
that upweight the lower tails of the distribution. 
$\wh{F}_\FIS$ will lead to lower-variance risk estimators 
for risk-seeking risk functionals that upweight the upper tails.

The RHS of the error bound in Lemma~\ref{lem:s_f_bernstein} 
gives a uniform confidence band on the estimated CDF 
that contains $F$ with high probability. 
However, the variance of $\wh{F}_\FIS$ or $\wh{F}_\SIS$ 
cannot be computed without knowledge of the underlying MDP. 
Instead, we can use the empirical variance
of $\wh{F}_\FIS$ or $\wh{F}_\SIS$:
\begin{lemma}[Empirical Bernstein Bound]\label{lem:s_f_empirical_bernstein}
    Let $\Var_n$ denote the sample variance. For any choice of $M$ points $\{t_j\}_{j=1}^M$ where $t_j \in [0, D]$, with probability at least $1-\delta$ we have that $\forall j \in [M]$,  
    \begin{align*}
        |F(t_j) - \wh{F}_\FIS(t_j)| &\leq \frac{c_1'\log(2M/\delta)}{n} + \sqrt{\frac{c_2'\Var_n[\wh{F}_\IS(t_j)]\log(2M/\delta)}{n}} , \\
        |F(t_j) - \wh{F}_\SIS(t_j)| &\leq \frac{c_1'\log(2M/\delta)}{n} + \sqrt{\frac{c\Var_n[\wh{S}_\IS(t_j)]\log(2M/\delta)}{n}}.
    \end{align*}
    \normalsize
\end{lemma}
Note that Lemma~\ref{lem:s_f_empirical_bernstein} 
is a bound for a finite number of points, 
while Lemma~\ref{lem:s_f_bernstein} 
is a uniform bound over all $t$. 
In practice, 
the sample variance of only a finite number of points
can be calculated. 

\subsection{Combining IS Estimators}\label{sec:s_vs_f_is}
Lemma~\ref{lem:s_f_bernstein} also indicates 
that the variance of the resulting estimator 
can be further reduced by  
combining the $F$ and $S$ estimators pointwise, 
such that the estimator with the lower variance is used at each $t$. 
Such a combined estimator may be especially effective 
for reducing the variance of the plug-in estimators 
for distortion risk functionals that place weight on both upper and lower tails of the distribution, 
such as the expected return. 

How should the choice between $\wh{F}_\FIS(t)$ or $\wh{F}_\SIS(t)$ be made for each $t$? 
Lemma~\ref{lem:s_f_empirical_bernstein} suggests that one effective method 
is to estimate the empirical variance of $\wh{F}_\FIS(t)$ and $\wh{F}_\SIS(t)$ at each $t$, 
and to simply use the estimate with lower empirical variance: 
\begin{equation}\label{eq:cis_estimator}
    \wh{F}_\CIS(t) = 
    \begin{cases}
        \wh{F}_\FIS(t),  &\;\text{if}\; \Var_n[\wh{F}_\FIS(t)] < \Var_n[\wh{F}_\SIS(t)], \\
        \wh{F}_\SIS(t), &\;\text{otherwise.}
    \end{cases}
\end{equation}
The error bound of such a combined estimator is given below.

\begin{proposition}\label{prop:combined}
    Let $\Var_n^{min}(t) = \min\{\Var_n[\wh{F}_\FIS(t)], \Var_n[\wh{F}_\SIS(t)]\}$. For any choice of $M$ points
    $\{t_j\}_{j=1}^M$ for $t_j \in [0,D]$ and $\delta \in (0, 1)$, define
    $$\Delta(\Var_n) := \sqrt{\frac{2\log(2M/\delta)}{n-1}} + \max_{j \in [M]}\Big|\Var_n[\wh{F}_\FIS](t_j) -  \Var_n[\wh{F}_\SIS](t_j)\Big|.$$    
    Then, for $n \geq \frac{8\log(2n/\delta)}{\Delta(\Var_n)^2}$, 
    we have with probability at least $1-\delta$ that $\forall j \in [M]$,  
    \begin{align*}
        |F(t_j) - \wh{F}_\CIS(t_j)| &\leq \frac{c_1''\log(2M/\delta)}{n} + \sqrt{\frac{c_2'' \Var_n^{min}(t_j)\log(2M/\delta)}{n}}.
    \end{align*}
    \normalsize
\end{proposition}

\section{Model-Based CDF Estimation}\label{sec:dm_cdf}
In the expected value OPE literature, 
model-based estimation can avoid the high variance 
associated with importance sampling 
by directly learning the expected return under $\pi$
using a model of the underlying \MDP 
learned from data \citep{jiang2016doubly}. 
Taking inspiration from such approaches,
we now develop the first model-based off-policy \CDF estimator.  
First, a model of the MDP $\ol{\M} = ( \Ss,\A, \ol{P}, \ol{R}, \gamma, \ol\mu )$ is formed from data, 
upon which a model $\ol{F}_{s_h, a_h}$ of the return distribution under $\pi$ 
can be directly computed using the recursion in~\eqref{eq:recursive_cdf} 
for all horizons $h$, states $s_h$, and $a_h$. 
Formally, for horizons $h=H,\ldots,1$, $\ol{F}$ 
can be computed recursively: 
\begin{align*}
    &\ol{F}_{s_h, a_h}(t) = \E_{\overline{P}, \overline{R}}\lb \ol{F}_{S_{h+1}}\lp\frac{t- R_h}{\gamma}\rp \rb \\ 
    &\ol{F}_{s_h}(t) = \E_\pi\lb \ol{F}_{s_h, A_h}(t)\rb
\end{align*}
with $\ol{F}^0_{s_{H+1}}(t) = \one\{ 0 \leq t \}$ for all $s_{H+1}$ at the end of the horizon. The models $\ol{F}$ can then be used for the \emph{direct method} (\DM) of estimation, which simply averages $\ol{F}$ for the starting states observed in the data: 
\begin{equation}\label{eq:dm_estimator}
    \wh{F}_\DM(t) = \frac{1}{n}\sum_{i=1}^n \ol{F}_{s_1^i}(t)
\end{equation}
Because this direct method of estimation directly deploys 
the target policy $\pi$ in the estimated model, 
the error of the estimators $\wh{F}_\DM$ and $\ol{F}$ 
are directly related to the misspecification of the MDP model $\ol{\mathcal{M}}$.  
The model $\ol{\mathcal{M}}$ and thus the direct estimator
is frequently prone to uncharacterizable and possibly high bias, 
e.g.,  when the dataset coverage is inadequate
or the state space is high dimensional.

\section{Doubly Robust CDF Estimation}\label{sec:dr_cdf}
We now define a new \emph{doubly robust} (DR) estimator 
for the distribution of returns 
in \MDP{}s (\S~\ref{sec:dr_estimator}), 
which takes advantage of both importance sampling (\S~\ref{sec:is_cdf}) 
and model-based (\S~\ref{sec:dm_cdf}) estimators 
to retain the unbiasedness of the IS \CDF estimator,
with potentially significant reduction in variance. 
In \S~\ref{sec:variance_lower_bound},
we derive the first Cramer-Rao lower bound 
on off-policy \CDF estimation, 
and show that the DR estimator 
achieves this lower bound
when the model $\ol{F}$ is equal to the true distribution.

\subsection{Doubly Robust (DR) Estimator}\label{sec:dr_estimator} 
\normalsize
The recursive formulation in~\eqref{eq:recursive_is_estimator} 
is key for defining the doubly robust \CDF estimator. 
At each time step $h$ in the recursive formulation, we use the model $\ol{F}$ as a baseline and apply an \IS-weighted data-dependent correction to obtain the doubly robust estimator. 
Specifically, the DR estimator has the following recursive form:
\normalsize
\begin{align}
    & \quad\quad\quad\wh{F}_\DR (t) = \frac{1}{n}\sum_{i=1}^n \wh{F}_{s_1^i}(t),\nonumber \\ 
    \text{where} \; &\wh{F}_{s_h^i}(t) = \ol{F}_{s_h^i}(t) + w(a_h^i,s_h^i) \lp \wh{F}_{s_{h+1}^i}\lp \frac{t - r_h^i}{\gamma}\rp - \ol{F}_{s_h^i,a_h^i} (t) \rp  \label{eq:dr_estimator}
\end{align}

\normalsize
The model $\ol{F}$ can be obtained using 
the methods described in \S~\ref{sec:dm_cdf}, and as before, $\wh{F}^0_{s_{H+1}}(t) = \one\llbrace 0 \leq t \rrbrace$ for all $s_{H+1}$.
Note that when $H=1$, 
\eqref{eq:dr_estimator} 
recovers the \DR estimator for contextual bandits 
previously derived in \cite{huang2021off}. 
Letting $z_h^i = \sum_{k=1}^h \gamma^{k-1} r_k^i$ be the return of trajectory $i$ up until step $h$, the recursion in~\eqref{eq:dr_estimator} can be unrolled over the $H$ timesteps of the trajectory to obtain 
\begin{align*}
    \wh{F}_\DR(t) =& \frac{1}{n}\sum_{i=1}^n \bigg(w_H^i \one\{z_H^i \leq t \} + \sum_{h=1}^H w_{h-1}^i \ol{F}_{s^i_h}\lp\frac{t - z_h^i}{\gamma^h}\rp - \sum_{h=1}^H w_h^i\ol{F}_{s^i_h, a^i_h}\lp\frac{t - z_h^i}{\gamma^h}\rp\Bigg).
\end{align*}

The bias and variance of the DR estimator is derived below. 
\begin{lemma}\label{lem:dr_bias_var}
    The doubly robust estimator is unbiased, and letting $\wh{F}_{s_h}$ be given as in~\eqref{eq:dr_estimator}, its variance is: 
    \small{
        \begin{align*}
            \Var_h^\beta \lb\wh{F}_{S_h}(t)\rb =& \E_h^\beta \lb \Var_h^\beta \lb w(A_h, S_h)\Delta_{S_h, A_h}(t) \Big\vert S_h \rb\rb\\ 
            &+ \E_{h}^\beta\lb w(A_h, S_h)^2 \Var_{h+1}\lb\wh{F}_{S_{h+1}}\lp \frac{t - R_h}{\gamma}\rp \Big\vert S_h, A_h \rb \rb \\
            &+ \Var_h\lb F_{S_h}(t)\rb 
        \end{align*}
    }%
    \normalsize
where $\Delta_{s,a}(t) = \ol{F}_{s,a}(t) - F_{s,a}(t)$. 

\end{lemma}

Comparing the variance expression above with that of the IS estimator (Lemma~\ref{lem:is_bias_var}), 
we see that the first term in the variance of the DR estimator contains a difference $\Delta_{s,a} = \ol{F}_{s,a} - F_{s,a}$, while the variance of the IS estimator contains only $F$. 
The variance in the first term is over stochastic actions given the current state. 

When $\Delta_{s,a} < F_{s,a}$, which is often the case in practice, the variance contributed by the first term can be significantly reduced in the \DR estimator. 
Consequently, the \DR estimator can be seen to reduce variance contributed by importance sampling for stochastic actions, but not variance from rewards or transitions, which arise from the last two terms in Lemma~\ref{lem:dr_bias_var}. 

To further reduce variance, taking inspiration from~\cite{thomas2015high}, we can combine weighted importance sampling with the DR estimator~\eqref{eq:dr_estimator}. Letting 
$\wt{w}_h = \sum_{i=1}^n w_h^i$
and $z_h^i = \sum_{k=1}^h \gamma^{k-1} r^i_k$, 
\begin{align*}
    \wh{F}_\WDR(t) =& \sum_{i=1}^n\Bigg( \frac{w^i_H}{\wt{w}_H}\one\llbrace z_H^i \leq t \rrbrace + \sum_{h=1}^H \frac{w^i_{h-1}}{\wt{w}_{h-1}} \ol{F}_{s^i_h}\lp\frac{t - z_h^i}{\gamma^h}\rp -\sum_{h=1}^H \frac{w^i_{h-1}}{\wt{w}_{h-1}}w(a^i_h, s^i_h)\ol{F}_{s^i_h, a^i_h}\lp\frac{t - z_h^i}{\gamma^h}\rp\Bigg).
\end{align*}
\normalsize

Neither the \DR nor \WDR estimators are guaranteed to valid CDFs; they may be outside the $[0, 1]$ interval or lose monotonicity due to the subtracted terms (particularly at lower sample sizes). In practice, we transform the \DR (and \WDR) estimators to valid \CDF{}s using: 
\begin{equation}\label{eq:mdr_estimator}
    \wh{F}_{\MDR}(t) = \text{Clip}\lp \max_{t' \leq t}\wh{F}_\DR(t'), 0, 1 \rp, 
\end{equation}
which simply does not allow the \CDF to decrease, 
and clips the value to the unit interval.

\subsection{Error Bound and Consistency}\label{sec:dr_error}
The $\WDR$ estimator is difficult to analyze, and we empirically demonstrate its efficacy in \S~\ref{sec:experiments}. 
The advantage of the \MDR estimator in~\eqref{eq:mdr_estimator}, however, is that it obeys the following finite sample error bound.
\begin{lemma}\label{lem:dr_error}
With probability at least $1-\delta$, 
\begin{align*}
    \left\Vert \wh{F}_\MDR - F \right\Vert_\infty \leq w_{max}^H\sqrt{\frac{72}{n}\log \frac{8n^{1/2}}{\delta}},
\end{align*}
and thus $\wh{F}_\MDR$ is uniformly consistent. 
\end{lemma}
We achieve a convergence rate of $O(w_{max}^H / \sqrt{n})$. 
However, we note that the above bound is conservative in the sense that it assumes worst-case model misspecification, under which the \DR estimator can perform worse than $\IS$. This contributes looseness to the result, and a tighter bound remains an open problem.

\subsection{Cramer-Rao Lower Bound}\label{sec:variance_lower_bound}
Though we have demonstrated the variance reduction potential of the DR estimator, a natural question to ask is whether there exists an unbiased estimator of the CDF that can even further reduce variance. 
Inspired by \cite{jiang2016doubly, huang2020importance}, 
we derive the first Cramer-Rao lower bound on the variance of off-policy \CDF estimation in MDPs, and show that the \DR estimator can achieve this lower bound.

\begin{theorem}\label{thm:lower_bound}
For unique directed acyclic graphs (Definition~\ref{def:dag}), the variance of any unbiased off-policy CDF estimator is pointwise lower bounded by 
$$
     \sum_{h=1}^H\E_{\Prob_\beta}\lb w_{h-1}^2 \Var_h \lb F_{S_h}\lp t - \sum_{k=1}^h R_k \rp \rb  \rb.
$$
\end{theorem}

This lower bound on the variance shows that difficulties introduced by the intrinsic state transition stochasticity of the underlying \MDP cannot be eliminated. 
In fact, the variance of the DR estimator in unique directed acyclic graphs is equal to the Cramer Rao lower bound when $\ol{F}$ is a perfect model, 
that is, $\ol{F}_{s_h, a_h}(t) = F_{s_h, a_h}(t)$
for all $s_h, a_h$ and $t$. Further, as the error between $\ol{F}$ and $F$ decreases, the variance approaches the Cramer-Rao lower bound.

\section{Risk Estimation}\label{sec:risk}
We have developed several methods for \CDF estimation and their error bounds, 
which we now translate to  off-policy risk assessment. 
Returning to Algorithm~\ref{algo:opra}, we have the CDF estimator $\wh{F}$ and the CDF estimation error $\epsilon$ (Line 2). 
We form the plug-in estimate of a risk functional $\rho$ as:  
\begin{equation}\label{eq:risk_plug_in}
    \wh\rho = \rho( \wh{F} )  
\end{equation}
In this section, we characterize 
the theoretical properties of the risk estimate $\wh\rho$
and show that the properties of the \CDF estimators 
translate directly to the risk estimators built upon them. 
Notably, the reduced variance of the \DR \CDF estimator 
results in risk estimators with reduced bias and variance, 
and for Lipschitz risk functionals (Definition~\ref{def:lipschitz}), the error of $\wh\rho$ is proportional $\epsilon$. For uniformly consistent CDF estimators, such as the previously introduced IS~\eqref{eq:is_estimator} and DR~\eqref{eq:dr_estimator} estimators, this implies that their downstream risk estimates will be consistent as well. 

\subsection{Bias, Variance, and Error}\label{sec:risk_bias_var}
In general, the risk estimator $\wh\rho$~\eqref{eq:risk_plug_in} can be biased even if the \CDF estimator is unbiased. 
Intuitively, however, $\wh\rho$ may have lower variance (and thus error) when it is estimated from a CDF estimate $\wh{F}$ with lower variance, because it is computed directly from $\wh{F}$ itself. 

For Lipschitz risk functionals, we formalize this intuition as a finite-sample convergence guarantee (Corollary~\ref{cor:lipschitz_error}), 
which demonstrates that the error of any Lipschitz risk estimator is upper bounded by its Lipschitz constant (from Definition~\ref{def:lipschitz}), and the error (upper bound) of the CDF estimate it utilizes.  
\begin{corollary}\label{cor:lipschitz_error}
For all Lipschitz risk functionals simultaneously, 
given a \CDF estimator $\wh{F}$ with error $\epsilon$, 
we have with probability at least $1-\delta$ that 
\begin{equation}\label{eq:confidence_general}
    |\rho(F) - \rho_p(\wh{F})| \leq L\epsilon,  
\end{equation}
where $L$ is the Lipschitz constant of $\rho$. Further, if $\wh{F}$ is uniformly consistent, then $\rho(\wh{F}) \stackrel{p}{\longrightarrow} \rho(F)$. 
\end{corollary}
As we have seen in Sections~\ref{sec:is_cdf} and~\ref{sec:dr_cdf}, the error guarantees for $\wh{F}$ are functions of its variance. 
Thus, Corollary~\ref{cor:lipschitz_error} demonstrates that we can expect to have faster convergence for risk estimates estimated on lower-variance CDFs. 
As an additional consequence, downstream risk estimators built from 
the \IS or \DR \CDF estimators, 
which have error that scale with $O(1/\sqrt{n})$, 
are guaranteed to be consistent. 

\subsection{Lower Bound}
We have previously demonstrated that the \IS \CDF estimator is minimax rate-optimal, 
and demonstrated that the $\DR{}\CDF$ estimator can achieve the variance lower bound for off-policy CDF estimation. 
However, the natural question to ask is: are their downstream risk estimators also rate-optimal? 
For general risk functionals this remains an open problem, 
but we show that for the Lipschitz risk functionals of $\cvara$ and expected return (which is equivalent to $\cvar_0$), 
their off-policy risk estimator $\wh \rho$ has the following lower bound.  
\begin{theorem}\label{thm:cvar_lb}
    Let $\F$ be the family of CDFs with bounded support and $\rho = \cvara$. For $\alpha \in [0, 1)$ and $n \geq c \frac{w_{max}^2}{\E_{\Prob_\beta}[w_H^2]}$ for a universal constant $c > 0$, 
    \begin{align*}
        \inf_{\wh\rho } \sup_{F^\pi \in \F} \E\lb \vert \wh\rho - \rho(F^\pi) \vert \rb \geq \frac{cD}{1-\alpha}\sqrt{\frac{\E_{\Prob_\beta}[w_H^2]}{n}}.
    \end{align*}
\end{theorem}
For large samples sizes, also matches existing results on off-policy mean estimation~\citep{li2015toward, wang2017optimal, ma2021minimax}. 
As the Lipschitz constants for $\cvara$ and expected return are 
$\frac{D}{1-\alpha}$ and $D$, respectively, 
the convergence rate of Corollary~\ref{cor:lipschitz_error} with \IS estimation is indeed minimax rate-optimal.

\begin{figure*}[t]
  \centering 
  \begin{subfigure}{0.77\textwidth}
      \centering 
      \hspace*{-5em}
      \includegraphics[trim={0 0 0 100}, clip, width=\textwidth]{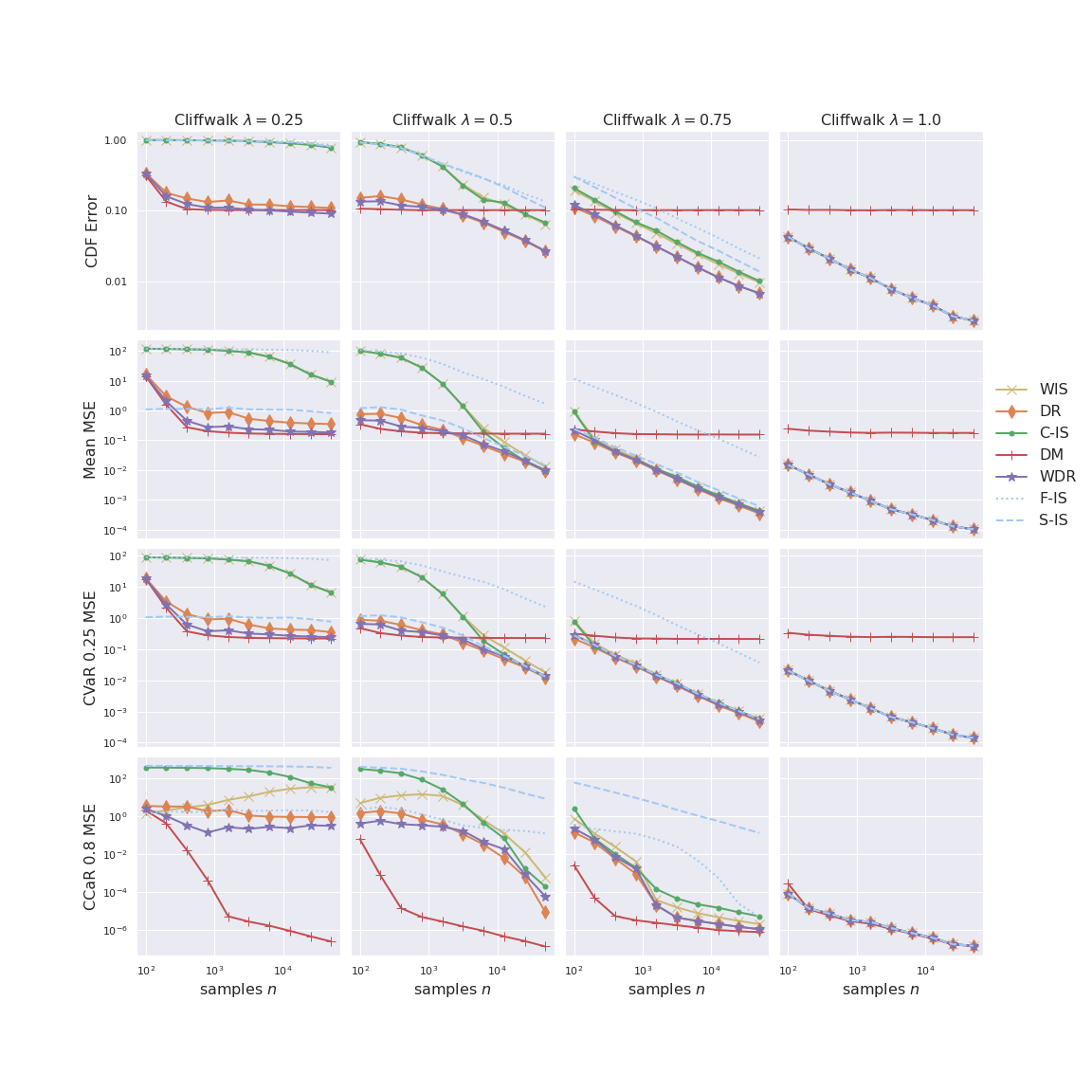}
  \end{subfigure}
  \begin{subfigure}{0.22\textwidth}
      \centering 
      \hspace*{-3em}
      \includegraphics[trim={0 0 0 100}, clip, width=\textwidth]{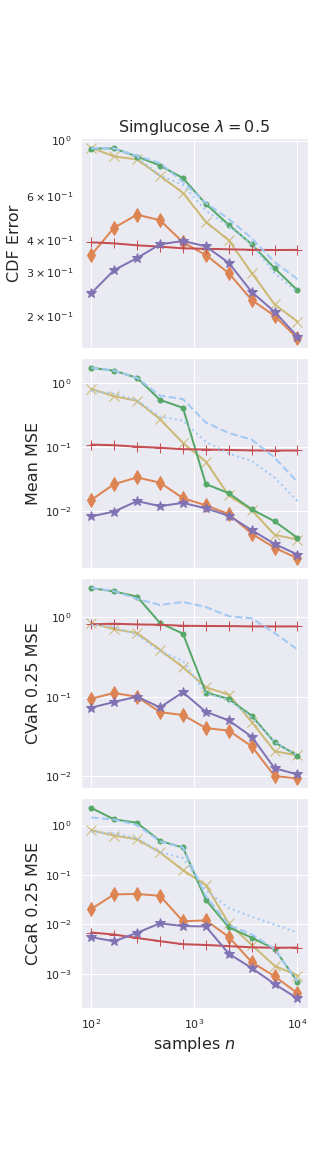}
  \end{subfigure}
  \vspace*{-4em}
  \caption{Sup-norm error of the CDF estimate \textbf{(row 1)}, and mean squared error of the expected return \textbf{(row 2)} and CVaR 0.25 \textbf{(row 3)} plug-in estimates for different $\lambda$ \textbf{(column)} in Cliffwalk and Simglucose. }
  \label{fig:error}
\end{figure*}

\section{Experiments}\label{sec:experiments}
We now provide empirical support for our theoretical results,  
and compare the performance of the \CDF and plug-in risk estimators in Figure~\ref{fig:error}.
We compare the \DR and \WDR estimators 
against the importance sampling estimators--- \FIS~\eqref{eq:is_estimator}, \SIS~\eqref{eq:s_estimator}, \CIS~\eqref{eq:cis_estimator}, \WIS ($\S$~\ref{sec:is_error})---and model-based \DM estimator. Our primary baselines of comparison are the \FIS and \WIS estimators, which are the only two off-policy CDF/risk estimators in existing work~\citep{chandak2021universal}. As we will see, the \DR and \WDR estimators outperform the IS-based estimators in all studied cases. 

\paragraph{Environment} We utilize two experimental domains: a tabular Cliffwalk~\citep{sutton2018reinforcement}, and continuous diabetes treatment simulator~\citep{xie2018simglucose}. The horizons of both are $H=200$, and the behavioral policy $\beta$ is a mixture between the target $\pi$ and a uniform policy $\text{UNIF}$, i.e., $\beta = \lambda \pi + (1-\lambda)\text{UNIF}$. Appendix~\ref{appendix:experiments} includes the full setup and additional experiments. 
Both environments have risk-related interpretations: in Cliffwalk the agent must take the shortest path while avoiding a slippery, high-cost cliff, while in Simglucose the agent must maintain a patient's glucose levels within a healthy range. 

\paragraph{Risk Functionals} As such, we estimate one each of risk-neutral, risk-averse, and risk-seeking risk functionals using OPRA: the expected return, $\cvara$, and a risk functional we call $\ccara$ (conditional cost-at-risk), respectively. While $\cvara$ is the expectation of the {worst-case $\alpha$ tail} of the distribution, CCaR is the expectation of the best-case $\alpha$ tail.

\paragraph{Results} In both environments, the \DR and \WDR estimators outperform the \IS estimators and their variants in CDF and risk estimation error (Figure~\ref{fig:error}). These environments present challenges for \IS estimators because of their long horizons, but \DR and \WDR effectively overcome them by incorporating model information. At very low sample sizes, however, the \DM estimator can outperform the other estimators.

Our results also demonstrate the variability of the \CDF estimator performance for different risk functionals; the gap between $\FIS$ and $\SIS$ is especially obvious. 
Under the Cliffwalk cost distribution, $\SIS$ far outperforms $\FIS$ for both the risk-neutral expected return and risk-sensitive $\cvara$, 
For the risk-seeking $\ccara$, however, $\FIS$ outperforms $\SIS$. The combined estimator $\CIS$ never does worse than either $\FIS$ or $\SIS$, and as $n$ increases can do better than both, as Proposition~\ref{prop:combined} implied. 
The $\DM$ estimator has low error under $\ccara$ because the model estimates the lower tail well, but the upper tail poorly. 
The \DR and \WDR estimator straddle best of both worlds; they perform well under all evaluated risk functionals.

\section{Discussion} 
In this work, we have introduced and analyzed variance-reduced off-policy CDF and risk estimators for MDPs, including the first doubly robust estimator for return CDF estimation.
Several future research directions are of interest. First, it is possible that estimators and bounds can be tailored and tightened for individual risk functionals, and a thorough comparison of different risk estimators is an important avenue of future work. 
Second, our results have also highlighted the risk estimation performance disparity between different types of CDF estimators. 
This disparity is more severe for risk-averse or risk-seeking risk functionals than expected value. 
As such, adaptive estimator selection (based on risk functionals of interest) may prove to be of both theoretical and practical interest. 
Lastly, in off-policy evaluation of the mean,  marginalized importance sampling~\citep{xie2019towards} has been shown to achieve the Cramer-Rao lower bound, and model-free regression~\citep{duan2020minimax} has been proven to be minimax optimal. 
One important future direction is to investigate whether these methods can be extended to off-policy CDF or risk estimation, which, unlike the mean, are nonlinear functions of the return.

\bibliographystyle{plainnat}
\bibliography{refs.bib}

\newpage 

\onecolumn 
\appendix

\section{Preliminaries: Risk Overview}\label{appendix:preliminaries}

Let $Z$ be a random variable. 
A \emph{risk functional} $\rho$
is a mapping from a space of random variables to the space of real numbers
$\rho: \mathcal{L}_\infty(\Omega, \F_Z,\Prob_Z) \rightarrow \R$.
Risk functionals are law-invariant if they depend
only on the distribution of $Z$~\citep{kusuoka2001law}.  
Formally $\rho(Z)$ is law invariant if for any pair of random variables $Z$ and $Z'$,
$
    F_Z = F_{Z'}\;\Longrightarrow\; \rho(Z)=\rho(Z'). 
$

Lipschitz risk functionals are a subset of law-invariatn risk functionals, and we 
now delineate popular classes of risks and their associated Lipschitz constants, where possible, adapted from \cite{huang2021convergence}.

\paragraph{Distorted Risk Functionals}
For $Z \geq 0$, distortion risk functionals are defined to be~\cite{denneberg1990distorted,wang1996premium,wang1997axiomatic,balbas2009properties}
\begin{align*}
    \rho(F_Z) = \int_0^\infty g(1-F_Z(t))dt =  \int_0^\infty g(S_Z(t))dt
\end{align*}
where the distortion $g:[0,1]\rightarrow[0,1]$ 
has $g(0)=0$ and $g(1)=1$, and is increasing. 
When $g(x)=x$, the expected value is recovered.
When $g(x) = \min\{\frac{x}{1-\alpha},1\}$ 
for $\alpha \in (0, 1)$,
CVaR at level $\alpha$ is recovered.

When $g(x)=F(F^{-1}(x)-F^{-1}(\alpha))$,
the Wang risk functional at level $\alpha$~\citep{wang1996premium}  is recovered, 
and the proportional hazard risk functional 
can by obtained by setting $g(x)=x^\alpha$ for $\alpha < 1$. 
Distorted risk functionals are coherent 
if and only if $g$ is concave \citep{wirch2001distortion}. 
Not all distorted risk functionals are coherent.
For example, setting $g(s)=\one_{\{s\geq 1-\alpha\}}$
recovers the value-at-risk (VaR),
which is not coherent. 

\paragraph{Coherent Risk Functionals} 
Coherent risk functionals satisfy 
properties called monotonicity, subadditivity,
translation invariance, and positive homogeneity  \cite{artzner1999coherent,delbaen2002coherent}.
Examples include expected value, 
conditional value-at-risk (CVaR), 
entropic value-at-risk,
and mean semideviation, proportional hazard, and Wang transform~\citep{chang2020risk, tamkin2019distributionally, tamar2015policy, shapiro2014lectures, wirch2001distortion, wang1996premium}. 

\begin{lemma}[Lipschitzness of Coherent and Distorted Risk Functionals \cite{huang2021off}]\label{lem:distorted_lipschitz}
On the space of random variables with support in $[0,D]$, 
the distorted risk functional of any $\frac{L}{D}$-Lipschitz distortion function 
is a $L$-Lipschitz risk functional.

The expected value risk functional is $D$-Lipschitz 
 and $\cvara$
 is $\frac{D}{\alpha}$-Lipschitz. 
\end{lemma}

\paragraph{Cumulative Prospect Theory (CPT) Risk Functionals} 
CPT risks~\cite{prashanth2016cumulative} are defined as:
\begin{align*}
    \rho(F_Z) = \int_0^{+\infty}g^+\left(1 - F_{u^+(Z)}(t)
    \right)dt - \int_0^{+\infty}g^-\left(1 - F_{u^-(Z)}(t)\right) dt, 
\end{align*}
where $g^+, g^-$: $[0, 1] \rightarrow [0, 1]$ 
are continuous 
with $g^{+/-}(0) = 0$ and $g^{+/-}(1) = 1$.
The functions $u^+, u^-: \R \rightarrow \R_+$ are continuous, 
with $u^+(z) = 0$ when $z \geq c$ and $u^-(z) = 0$ 
when $z < c$ for some constant $c \in \R$.
The CPT functional handles gains and losses separately, applying distortion $g^+$ to the part of the distribution corresponding to gains (controlled by $u^+$), and $g^-$ to the part corresponding to losses (controlled by $u^-$).

\begin{lemma}[Lipschitzness of CPT Functional]\label{lem:cpt_lipschitz}
If the CPT distortion functions $g^+$ and $g^-$
are both $\frac{L}{D}$-Lipschitz, 
then the CPT risk functional is $L$-Lipschitz.
\end{lemma}

\paragraph{Other Risk Functionals} 
Many other risk functionals, 
though not members of a specific class, 
can be shown to be Lipschitz. 
This includes the variance and mean-variance.
The variance is defined as 
$\rho(F_Z) = 2\int_{0}^\infty t(1-F_Z(t))dt - \left(\int_{0}^\infty (1-F_Z(t))dt \right)^2.$

\begin{lemma}[Lipschitzness of Variance]\label{lemma:variance_lipschitz}
 On the space of random variables with support in $[0,D]$,
 variance is a $3D^2$-Lipschitz risk functional. For mean-variance given 
by $\rho(Z) = \E[Z] + \lambda \Var(Z)$ 
for some $\lambda > 0$, it is $(1 + 3\lambda D^2)$-Lipschitz. 
\end{lemma}

\section{Proofs for Importance Sampling CDF Estimation (Section~\ref{sec:is_cdf})}\label{appendix:is}

\subsection{Bias and Variance (Proof of Lemma~\ref{lem:is_bias_var}) }
As the IS estimator can be seen as a special case of the DR estimator with $\ol{F} = 0$, we defer the proof of bias and variance to Appendix~\ref{appendix:dr_bias_var}.

\subsection{Error Upper Bound (Proof of Lemma~\ref{lem:is_error})}

The proof of the upper bound closely follows the proof of Theorem 5.1 from \cite{huang2021off}. 

For any $\lambda > 0$, we have 
\begin{align*}
    \lambda \E\lb \Vert \wh{F} - F \Vert_\infty \rb =& \E\lb \log \exp \lp \lambda \Vert \wh{F} - F \Vert_\infty \rp \rb \\ 
    &\leq \log\E\lb \exp \lp \lambda \Vert \wh{F} - F \Vert_\infty \rp \rb 
\end{align*}

To upper bound the RHS, define the following function class: 
\begin{align*}
    \mathbb{F}(n):= \llbrace f(g) :=\varrho\frac{1}{n}\one_{\llbrace g\leq t\rrbrace}: \forall t\in\R;\forall g\in\mathbb{Q}, \varrho\in\llbrace -1,+1\rrbrace \rrbrace
\end{align*}

Note that this is a countable set. Denote $w(\tau) = \prod_{h=1}^H w(a_h, s_h)$, and $g^i = \sum_{h=1}^H \gamma^h r_h^i$ the return of trajectory $i$ for short. Using this definition, we have
\begin{align*}
    \sup_{t\in\R} \lv \wh F_\IS(t) -  F(t) \rv =\sup_{f\in\mathbb{F}(n)} \lv \frac{1}{n}\sum_{i=1}^n  w(\tau^i)f(g^i) - \E_{\Prob_\beta}\lb \frac{1}{n}\sum_{i=1}^n w(\tau^i)f(g^i) \rb  \rv 
\end{align*}

Then adapting the proof of Theorem 5.1 in \cite{huang2021off}, we have for any $\lambda \in (0, \frac{n}{2w_{max}})$ that,
\begin{align*}
    \E\lb \exp \lp \lambda \Vert \wh{F} - F \Vert_\infty \rp \rb &\leq  4\exp\left(\frac{n\lambda^2\frac{4 \E_{\Prob_\beta}[w(\tau)^2])}{n^2}}{2\left(1-\lambda\frac{2}{n}w_{max}\right)}\right) 
\end{align*}

Then 
\begin{align}
    \E\lb \Vert \wh{F} - F \Vert_\infty \rb &\leq \frac{1}{\lambda}\log \lp 4\exp\left(\frac{n\lambda^2\frac{4 \E_{\Prob_\beta}[w(\tau)^2])}{n^2}}{2\left(1-\lambda\frac{2}{n}w_{max}\right)}\right) \rp \nonumber \\  
    &\leq \frac{\log 4}{\lambda} + \frac{2\lambda \E_{\Prob_\beta}[w(\tau)^2]}{n(1 - \lambda\frac{2}{n}w_{max})}\label{eq:1.1/upper_bound}
\end{align}

Since this holds for any $\lambda \in (0, \frac{n}{2w_{max}})$, it also holds for $\min_\lambda$ of the RHS. To find the minimizer, we solve 
\begin{align*}
    \frac{d}{d\lambda} \left( \frac{\log 4}{\lambda} + \frac{2\lambda \E_{\Prob_\beta}[w(\tau)^2]}{n - 2\lambda w_{max}} \right) = 0
\end{align*}
which WolframAlpha tells us is solved by 
\begin{align*}
    \lambda^* = \frac{n}{\sqrt{\E_{\Prob_\beta}[w(\tau)^2] n \log 2} + 2w_{max}}
\end{align*}
if $\E_{\Prob_\beta}[w(\tau)^2] n \neq w_{max}^2\log 16$. We observe that $\lambda^* \in (0, \frac{n}{2w_{max}})$ satisfies the constraints on $\lambda$. 
Plugging $\lambda^*$ back into~\eqref{eq:1.1/upper_bound} we have, 
\begin{align*}
    \E\lb \Vert \wh{F} - F \Vert_\infty \rb &\leq \frac{\log 4}{\frac{n}{\sqrt{\E_{\Prob_\beta}[w(\tau)^2] n \log 2} + 2w_{max}}} + \frac{ \frac{2n \E_{\Prob_\beta}[w(\tau)^2]}{\sqrt{\E_{\Prob_\beta}[w(\tau)^2] n \log 2} + 2w_{max}}}{n - \frac{2w_{max}n}{\sqrt{\E_{\Prob_\beta}[w(\tau)^2] n \log 2} + 2w_{max}}} \\ 
    &= \log 4 \sqrt{\log 2} \sqrt{\frac{\E_{\Prob_\beta}[w(\tau)^2]}{n}} + \frac{2\log 4 w_{max}}{n} +  2\sqrt{\frac{\E_{\Prob_\beta}[w(\tau)^2]}{n\log 2}} \\ 
    &\leq c_1\sqrt{\frac{\E_{\Prob_\beta}[w(\tau)^2]}{n}} + c_2 \frac{w_{max}}{n}
\end{align*}

\subsection{CDF Estimation Lower Bound (Proof of Theorem~\ref{thm:cdf_lower_bound})}
We begin by deriving a CDF estimation lower bound in multi-armed bandits (MABs). We can reduce the estimation problem in the MDP case to the MAB case by treating entire trajectories $\tau$ as ``actions". 

Let a multi-armed bandit be defined by $K$ arms, and suppose the behavioral policy is given by $\beta(a)$ for $a \in [K]$, and the target policy is $\pi(a)$ for $a \in [K]$, with importance weight $w(a) := \beta(a) / \pi(a)$. 
Let $w_{max} = \max_a w(a)$. 

\begin{lemma}[CDF Lower Bound in MABs]\label{lem:mab_cdf_lb}
When $n \geq c\frac{w_{max}^2}{\sum_a \beta(a)w^2(a)}$ for a universal constant $c$, the minimax lower bound on off-policy CDF estimation in MABs is
\begin{equation}
    \inf_{\wh F } \sup_{F \in \F} \E\lb \Vert \wh{F} - F^\pi \Vert_\infty \rb \geq c\sqrt{\frac{\sum_{a \in [K]}\beta(a)w^2(a)}{n}}.  
\end{equation}
\end{lemma}

We now lift this result to the MDP case. For a horizon of length $H$, the space of possible actions of size $|\tau| = |\Ss|^{H+1} |\A|^{H}$. 
Recall that in MDPs, $w_{max} = \max_\tau w(\tau)$, and $w(\tau) = \frac{\pi(\tau)}{\beta(\tau)} = \prod_{h=1}^H \frac{\pi(a_h|s_h)}{\beta(a_h|s_h)}$. Thus using Lemma~\ref{lem:mab_cdf_lb}, when 
$
    n \geq c \frac{w_{max}^2}{\sum_{\tau}\Prob_\beta(\tau) w^2(\tau)}, 
$
we have 
\begin{align}
    \inf_{\wh F } \sup_{F \in \F} \E\lb \Vert \wh{F} - F^\pi \Vert_\infty \rb \geq c\sqrt{\frac{\sum_{\tau}\Prob_\beta(\tau)w^2(\tau)}{n}}.
\end{align}

This matches our Bernstein-style bound on the error of IS estimation up to a constant: 
\begin{align}
    \E\lb \Vert \wh{F} - F \Vert_\infty \rb \leq c_1\sqrt{\frac{\sum_{\tau}\Prob_\beta(\tau)w^2(\tau)}{n}} + c_2 \frac{w_{max}}{n}. 
\end{align}

\begin{proof}[Proof of Lemma~\ref{lem:mab_cdf_lb}]\label{proof:mab_cdf_lb}
    We use Le Cam's method to derive the lower bound. 
    We first need to select two problem MAB problem instances, by defining their reward distributions $R_1$, $R_2$. 
    Then define distributions $P_1 = \beta \circ R_1$, and $P_2 = \beta \circ R_2$. 
    
    Intuitively, we want to select them so that the difference $\Vert F_1^\pi - F_2^\pi \Vert_\infty$ is large enough that an estimator can distinguish between them, but with the distance $\dkl(P_1 \Vert P_2)$ to be small enough that the testing problem is not too easy. 
    
    For each action $a \in [K]$, let $R_2(\cdot|a)$ be a Bernoulli distribution on the set $\{0, 1\}$ with probability $p = 1/2$, and let $R_1(\cdot|a)$ be a Bernoulli distribution with $p = 1/2 + \delta(a)$ for some $\delta(a) \in [0, 1/2]$. Then 
    \begin{align*}
        \Vert F_1^\pi - F_2^\pi \Vert_\infty &= \max_{t \in \llbrace 0, 1 \rrbrace} \left\vert F_1^\pi(t) - F_2^\pi(t) \right\vert\\ 
        &= \left\vert F_1^\pi(t = 0) - F_2^\pi(t = 0) \right\vert \\
        &= \left\vert \sum_a \pi(a) \delta(a) \right\vert 
    \end{align*}
    
    At the same time, 
    \begin{align*}
        \dkl(P_1 \Vert P_2) &= \sum_a \beta(a) \dkl(R_1 \Vert R_2) \leq 4 \sum_a \beta(a)\delta^2(a)
    \end{align*}
    Since we want the probability of error in our test to be a constant, we want to upper bound the above by $\frac{1}{2n}$. 
    Thus to obtain as tight a lower bound as possible, we want to solve the optimization problem 
    \begin{equation}
        \max_{\delta \in \R^K}  \sum_a \pi(a) \delta(a)  \quad\text{s.t.}\quad 4 \sum_a \beta(a)\delta^2(a) \leq \frac{1}{2n},\; \delta \in [0, 1/2]
    \end{equation}
    
    Using Lemma 3 from \cite{ma2021minimax}, the Fenchel dual of the optimization problem in the above panel is 
    \begin{equation}\label{primal}
        \min_{\nu \in \mathbb{R}^K} \sqrt{\frac{1}{8n}\sum_a \frac{(\pi(a) - \nu(a))^2}{\beta(a)}} + \frac{1}{2}\sum_a |\nu(a)|
    \end{equation}
    
    Let $\nu^*$ be the solution to~\eqref{primal} and define the set of actions 
    \begin{equation*}
        S^* = \{a \in [K] | \nu^*(a) > 0\}
    \end{equation*}
    
    Then Lemma 2 from \cite{ma2021minimax} states that
    \begin{align*}
        \min_{\nu \in \mathbb{R}^K} \sqrt{\frac{1}{8n}\sum_a \frac{(\pi(a) - \nu(a))^2}{\beta(a)}} + \frac{1}{2}\sum_a |\nu(a)| \asymp \pi(S^*) + \sqrt{\frac{\sum_{a \notin S^*}\beta(a)w^2(a)}{n}}
    \end{align*}
    
    Thus we have that for some universal constant $c$, 
    \begin{equation}
        \inf_{\wh F } \sup_{F \in \F} \E\lb \Vert \wh{F} - F^\pi \Vert_\infty \rb \geq c\llbrace \pi(S^*) + \sqrt{\frac{\sum_{a \notin S^*}\beta(a)w^2(a)}{n}}\rrbrace 
    \end{equation}
    
    When we have sample size $n$ such that $n \geq c\frac{w_{max}^2}{\sum_a \beta(a)w^2(a)}$, $S^* = \emptyset$ and the lemma statement follows. 
\end{proof}

\subsection{Proofs for Section~\ref{sec:s_f_estimator}} 

\subsubsection{Bernstein Bounds for \texorpdfstring{$S$}{S} and \texorpdfstring{$F$}{F} (Proof of Lemma~\ref{lem:s_f_bernstein})}
\begin{definition}[Bounded Variation]
    A function $f$ has bounded variation if there exists $C < \infty$ such that 
    \[ 
        \sup_{N, \lbrace t_i\rbrace_{i=1}^N}\sum_{i=1}^N|f(t_{i+1})-f(t_i)| \leq C
    \]
\end{definition}

\begin{assumption}\label{assum:bounded_var}
    Suppose that the random variable of returns takes values on a compact set. Assume the variance of the IS estimator $\Var^{\wh{F}}(t)$ has bounded variation for all $t \in \mathbb{R}$.  
\end{assumption}

\begin{lemma}~\label{lem:bounded_var}
    If a function $f$ has bounded variation, then there exists monotone functions $f^+$, $f^-$ with bounded variation such that $\forall t$, 
    \[ 
        f(t) = f^+(t) - f^-(t)
    \]
\end{lemma}

\begin{proof}
Using Lemma~\ref{lem:bounded_var}, let $\Var^{\wh{F}} = \Var^+ - \Var^-$, and note that $\Var^+, \Var^- \in [0, d_\infty]$ and are monotone functions. As they are monotone functions, let $\{s^+_j\}_{j=1}^M$ and $\{s^-_j\}_{j=1}^M$ be the respective stepping points for their stepwise approximations. Further, let $\{s^F_k\}_{k=1}^{M'}$ be the stepping points for the CDF $F$. 
Finally, overriding notation slightly, let $\{s_j\}_{j=1}^{M' + 2M}$ be the union of these three sets of stepping points.   
We know that between adjacent stepping points, $F$, $\Var^+$, and $\Var^-$ do not change much. Let $\zeta^F, \zeta^+, \zeta^-$ be the step functions constructed using the stepping points for each function, respectively, and let $\zeta^{\Var^{\wh{F}}} = \zeta^+ - \zeta^-$. Formally, this means that
\begin{align*}
    \sup_{t \in \mathbb{R}}|F(t) - \zeta^F(t)| &\leq 1/2M' \\ 
    \sup_{t \in \mathbb{R}}|\Var^+(t) - \zeta^+(t)| &\leq 1/2M \\ 
    \sup_{t \in \mathbb{R}}|\Var^-(t) - \zeta^-(t)| &\leq 1/2M \\
    \sup_{t \in \mathbb{R}}|\Var^{\wh{F}}(t) - \zeta^{\Var^{\wh{F}}}(t)| &\leq 1/M 
\end{align*}

At a given $s_j$, Bernstein's inequality gives us 
\begin{align*}
    |F(s_j)-\wh F(s_j)| &\leq \frac{c_1\log(1/\delta)}{n} + \sqrt{\frac{c_2\Var^{\wh{F}}(s_j)\log(1/\delta)}{n}} \\ 
    &= \frac{c_1\log(1/\delta)}{n} + \sqrt{\frac{c_2\lp \Var^+(s_j) - \Var^-(s_j)\rp \log(1/\delta)}{n}}
\end{align*}

Taking a union bound, this holds for all $\{s_j\}_{j=1}^{M' + 2M}$ with probability at least $1 - \delta(M' + 2M)$. 

For each $j$ and $t \in [s_j, s_{j+1}]$, if $F(t) \geq \wh{F}(t)$, 
\begin{align*}
    F(t)-\wh F(t)&\leq F(s_{j+1})-\wh F(s_j)\\
    &\leq \frac{1}{M'}+F(s_{j})-\wh F(s_j)\\
    &\leq \frac{1}{M'} + \frac{c_1\log(1/\delta)}{n} + \sqrt{\frac{c_2\Var^{\wh{F}}(s_j)\log(1/\delta)}{n}}\\
    &\leq \frac{1}{M'} + \frac{c_1\log(1/\delta)}{n} + \sqrt{\frac{c_2\Var^{\wh{F}}(t)\log(1/\delta)}{n}}+\sqrt{\frac{c_2\log(1/\delta)}{Mn}}
\end{align*}

If $\wh{F}(t) > F(t)$, 
\begin{align*}
    \wh{F}(t) - F(t) &\leq \wh{F}(s_{j+1}) - F(s_j) \\ 
    &\leq \frac{1}{M'} + |\wh{F}(s_{j+1}) - F(s_{j+1})| \\ 
    &\leq \frac{1}{M'} + \frac{c_1\log(1/\delta)}{n} + \sqrt{\frac{c_2\Var^{\wh{F}}(s_{j+1})\log(1/\delta)}{n}}\\ 
    &\leq \frac{1}{M'} + \frac{c_1\log(1/\delta)}{n} + \sqrt{\frac{c_2\Var^{\wh{F}}(t)\log(1/\delta)}{n}} + \sqrt{\frac{c_2\log(1/\delta)}{Mn}}
\end{align*}
Choosing $M, M' \propto \sqrt{n}$ gives the resulting bound. 

Similarly, for $\wh{F}_{\SIS}$, we have 
\begin{align*}
    F - \wh{F}_\SIS = 1 - S - (1 - \wh{S}_\IS) = \wh{S}_\IS - S
\end{align*}

The remainder of the proof occurs in the same manner, but with $\Var[\wh{S}_\IS]$ on the RHS. 
\end{proof}

\subsubsection{Empirical Bernstein Bound (Proof of Lemma~\ref{lem:s_f_empirical_bernstein})}
\begin{proof}
We begin with the Bernstein bounds for $\wh{F}_\FIS(t_j)$, $\wh{F}_\SIS(t_j)$ at a given $t_j$. That is, 
\begin{align*}
    |F(t_j)- \wh{F}_\FIS(t_j)| &\leq \frac{c_1\log(1/\delta)}{n} + \sqrt{\frac{c_2\Var[\wh{F}_\IS(t_j)]\log(1/\delta)}{n}}
\end{align*}
and a similar statement holds for $\wh{F}_\SIS$. From \cite{maurer2009empirical}, we have the following bound on estimation of the variance:
\begin{lemma}[Theorem 10 from~\cite{maurer2009empirical}]\label{lem:empirical_variance}
With probability at least $1-\delta$, for random variables $X \in [0, 1]$, 
\begin{equation}
    \lv \sqrt{\Var_n} - \sqrt{\Var} \rv \leq \sqrt{\frac{2\ln 1/\delta }{n-1}}
\end{equation}
\end{lemma}

Taking a union bound between Lemma~\ref{lem:empirical_variance} and Lemma~\ref{lem:s_f_bernstein} at all $M$ points, we have that with probability at least $1-\delta$ that for all $t_j$, 
\begin{align*}
    \lv F(t_j) - \wh{F}_\FIS(t_j) \rv \leq& \frac{c_1\log(2M/\delta)}{n} + \sqrt{\frac{c_2 \Var_n[\wh{F}_\IS(t)]  \log(2M/\delta)}{n}} \\
    &+ \sqrt{\frac{2c_2 \log^2(2M/\delta)}{n(n-1)}} +  \sqrt{\frac{c_2\log(2M/\delta)}{Mn}}
\end{align*}

 Similar steps give the bound for $\wh{F}_\SIS$, with $\Var_n[\wh{S}_\IS(t)]$ on the RHS. 

\end{proof}

\subsubsection{Empirical Bernstein Bound for Combined Estimator (Proof of Proposition~\ref{prop:combined})}

In order to correctly choose either $\wh{F}_\FIS$ or $\wh{F}_\SIS$ for each $t_j$, we need that for all $t_j$, the empirical variance estimates $\Var_n$ are separated by at least $\epsilon_\Var(t) = \frac{1}{2}|\Var[\wh{F}_\FIS(t)] - \Var[\wh{F}_\SIS(t)]| \leq \frac{1}{2}|\Var_n[\wh{F}_\FIS(t)] - \Var_n[\wh{F}_\SIS(t)]| + \sqrt{\frac{\ln 1/\delta'}{2(n-1)}}$. For this to be the case, we set 
\begin{align*}
    \delta(t) = \frac{1}{2}\exp{\lp -\frac{1}{2}(n-1)\epsilon_\Var(t)^2 \rp}, 
\end{align*}
and $\delta = \max_{j}\delta(t_j)$. Thus if 
\begin{align*}
    n \geq \frac{8\log (2n/\delta)}{\max_{j}\epsilon_\Var(t_j)^2}, 
\end{align*}
then we correctly choose $\wh{F}_\FIS$ or $\wh{F}_\SIS$ for each $t_j$, which gives the result.

\section{Proofs for Doubly Robust CDF Estimation (Section~\ref{sec:dr_cdf})} 

\subsection{Bias and Variance (Proof of Lemma~\ref{lem:dr_bias_var})}\label{appendix:dr_bias_var}
We will prove the unbiasedness of the DR estimator by induction. 
As the base case, $F_{s_{H+1}}^{0}(t) = 0$ for all states $s$ and for all $t$, which is unbiased by definition. 

As the induction assumption, we assume that the DR estimator at the step $h+1$ is unbiased, that is  $\E_{h+1}[\wh{F}_{S_{h+1}}(t)] = \E_{h+1}[F_{S_{h+1}}](t)$. 

We need to prove that it is also unbiased for the $h$th step. Recall for now we assume that rewards are deterministic. Then for any $t$: 
\allowdisplaybreaks
\begin{align*}
    \E^\beta_h[\wh{F}_{S_{h}}(t)] &= \E_{\Prob_\beta}\left[ \ol{F}^{H+1-h}_{S_h}(t) +  w(A_h,S_h) \lp \wh{F}^{H-h}_{S_{h+1}}\lp \frac{t - R(S_h, A_h)}{\gamma}\rp - \ol{F}^{H+1-h}_{S_h,A_h} (t) \rp \right] \\
    &= \E^\beta_{h}\lb \ol{F}^{H+1-h}_{S_h}(t) - w(A_h, S_h)\ol{F}^{H+1-h}_{S_h,A_h} (t) \rb + \E^\beta_{h} \lb w(A_h,S_h) \wh{F}^{H-h}_{S_{h+1}}\lp \frac{t - R(S_h, A_h)}{\gamma}\rp \rb\\
    &= \E_{h}\lb \ol{F}^{H+1-h}_{Z(S_h)}(t) -  \E_{h}\lb \ol{F}^{H+1-h}_{S_h,A_h} (t) \Big\vert S_h \rb \rb \\ 
    &\quad+ \E_{h}\left[ \wh{F}^{H-h}_{S_{h+1}}\lp \frac{t - R(S_h, A_h)}{\gamma} \rp \right] \\
    &= \E_{h}\lb \ol{F}^{H+1-h}_{Z(s_h)}(t) - \ol{F}^{H+1-h}_{Z(s_h)}(t) \right] + \E_{h}\lb \E_{h+1}\lb \wh{F}^{H-h+1}_{S_h+1}\lp \frac{t - R(S_h, A_h)}{\gamma} \rp \rb \rb\\ 
    &= \E_{h}\lb \E_{h+1}\lb F^{H-h+1}_{S_h+1}\lp \frac{t - R(S_h, A_h)}{\gamma} \rp \rb \rb\\ 
    &= \E_h\lb F_{S_h}(t)\rb
\end{align*}
where the third equality incorporates the importance sampling weight $w$ into a change of measure from $\beta$ to $\pi$, the foruth equality uses the conditional expectation. The second to last equality uses the induction assumption and the last uses the recursive identity. 

Next, we derive the variance of the DR estimator. 
\begin{align*}
    \Var_h^\beta \lb\wh{F}_{S_h}(t)\rb =& \E_h^\beta\lb 
    \lp\wh{F}_{S_h}(t)\rp^2 \rb - \E_h^\beta\lb 
    \wh{F}_{S_h}(t) \rb^2 \\ 
    =& \E_h^\beta\lb \lp \ol{F}^{H+1-h}_{S_h}(t) +  w(A_h,S_h) \lp \wh{F}^{H-h}_{S_{h+1}}\lp \frac{t - R(S_h, A_h)}{\gamma}\rp - \ol{F}^{H+1-h}_{S_h,A_h} (t) \rp \rp^2 - F_{S_h}(t)^2 \rb \\
    &+ \E_h^\beta\lb F_{S_h}(t)^2 \rb - \E_h^\beta\lb 
    F_{S_h}(t) \rb^2\\ 
    =& \E_h^\beta\lb \lp \ol{F}^{H+1-h}_{S_h}(t) +  w(A_h,S_h) \lp \wh{F}^{H-h}_{S_{h+1}}\lp \frac{t - R(S_h, A_h)}{\gamma}\rp - \ol{F}^{H+1-h}_{S_h,A_h} (t) \rp \rp^2 - F_{S_h}(t)^2 \rb \\
    &+ \Var_h\lb F_{S_h}(t)\rb \\ 
    =& \E_h^\beta\bigg[\bigg( \ol{F}^{H+1-h}_{S_h}(t) + w(A_h, S_h)F_{S_h, A_h}(t) - w(A_h, S_h)\ol{F}^{H+1-h}_{S_h, A_h} \\
    &+ w(A_h,S_h) \lp \wh{F}_{S_{h+1}}\lp \frac{t - R(S_h, A_h)}{\gamma}\rp - F_{S_h,A_h} (t) \rp \bigg)^2 - F_{S_h}(t)^2 \bigg] + \Var_h\lb F_{S_h}(t)\rb \\ 
    =& \E_h^\beta\bigg[\bigg( \ol{F}^{H+1-h}_{S_h}(t) - w(A_h, S_h)\Delta_{S_h, A_h}(t) \\
    &+ w(A_h,S_h) \lp \wh{F}_{S_{h+1}}\lp \frac{t - R(S_h, A_h)}{\gamma}\rp - \E_{h+1}\lb \wh{F}_{S_{h+1}} \lp \frac{t - R(S_h, A_h)}{\gamma}\rp \rb \rp \bigg)^2 - F_{S_h}(t)^2 \bigg] \\
    &+ \Var_h\lb F_{S_h}(t)\rb \\ 
    =& \E_h^\beta\lb \lp \ol{F}^{H+1-h}_{S_h}(t) - w(A_h, S_h)\Delta_{S_h, A_h}(t) \rp^2 - F_{S_h}(t)^2 \rb \\
    &+ \E_h \lb \lp w(A_h,S_h) \lp \wh{F}_{S_{h+1}}\lp \frac{t - R(S_h, A_h)}{\gamma}\rp - \E_{h+1}\lb \wh{F}_{S_{h+1}} \lp \frac{t - R(S_h, A_h)}{\gamma}\rp \rb \rp \rp^2 \rb \\
    &+ \Var_h\lb F_{S_h}(t)\rb \\ 
    =& \E_h^\beta \lb \Var_h^\beta \lb w(A_h, S_h)\Delta_{S_h, A_h}(t) \Big\vert S_h \rb\rb + \E_{h}^\beta\lb w(A_h, S_h)^2 \Var_{h+1}\lb\wh{F}_{S_{h+1}}\lp \frac{t - R(S_h, A_h)}{\gamma}\rp \Big\vert S_h, A_h \rb \rb \\
    &+ \Var_h\lb F_{S_h}(t)\rb 
\end{align*}
where $\Delta_{s,a}(t) = \overline{F}_{s,a}(t) - F_{s,a}(t)$. 

\subsection{Error Bound on Doubly Robust Estimator (Proof of Lemma~\ref{lem:dr_error})}
We separate the bound into two parts using the definition of $\wh{F}_\DR$: 
\begin{align}
    \lV F - \wh{F}_\DR \rV_\infty \leq& \lV F - \frac{1}{n} \sum_{i=1}^n w(\tau^i) \one\llbrace \sum_{h=1}^H \gamma^h r_h \leq t \rrbrace \rV_\infty \nonumber \\
    &+ \lV \frac{1}{n}\sum_{i=1}^n w(\tau_{1:h-1}^i) \lp w(s^i_h, a^i_h)\ol{F}_{S_h, A_h}\lp \frac{t - \sum_{k=1}^h R_k^i}{\gamma^h} \rp  - \ol{F}_{S_h}\lp \frac{t - \sum_{k=1}^h R_k^i}{\gamma^h} \rp \rp\rV_\infty \label{term_to_bound}
\end{align}
Lemma~\ref{lem:is_error} upper bounds the first term, and we are left to bound the second. 

Let $\tau = (s^0, a^0, r^0, \ldots, s^h, a^h, r^h, \ldots, s^H, a^H, r^H, s^{H+1})$ be a trajectory, and define $w(\tau_{1:h}) = \prod_{k=1}^h w(s_h, a_h)$. Define the function class 
\small{
\begin{align*}
    \mathbb{F}(m, w) = \llbrace f\lp \{\{x_{j}^h\}_{j=1}^{M}\}_{h=1}^{H}, \tau \rp := \frac{\varrho}{m} \sum_{h=1}^H w(\tau_{1:h-1}) \sum_{j=1}^m w(a^h, s^h)\one_{\{x_j^h \leq t\}} : \forall t \in \R, \varrho \in \{-1, +1\}, \{x_{j}^h\}_{j,h=1}^{M,H} \in \mathbb{Q}^{M \times H}\rrbrace
\end{align*}
}%
\normalsize

Note that~\eqref{term_to_bound} is 
\begin{align*}
    \leq \sup_{\zeta \in \mathbb{F}(m,w)}\lv \frac{1}{n}\sum_{i=1}^n \zeta\lp\{x^j_{S_i^h, A_i^h}\}, \tau_i\rp - \frac{1}{n}\sum_{i=1}^n\E_{\Prob_\beta}\lb \zeta\lp\{x^j_{S_i^h, A_i^h}\}, \tau_i\rp \Big| \{ S_i^h \}_{h=1}^H \rb  \rv + \frac{1}{m}
\end{align*}

Going forward, we refer to $\zeta\lp\{x^j_{S_i^h, A_i^h}\}, \tau\rp$ as $\zeta(\tau_i)$ for short. For $\lambda > 0$ we have 
\begin{align*}
    &\E_{\Prob_\beta}\left[\exp\left(\lambda\sup_{\zeta \in \mathbb{F}(m, w)} \left( \frac{1}{n}\sum_{i=1}^n \zeta(\tau_i) -  \frac{1}{n}\sum_{i=1}^n \E_{\Prob_\beta}\left[ \zeta(\tau_i) \Big\vert \{S^h_i\}_{h=1}^H \right]\right) \right) \right]\\
    &=\E_{\Prob_\beta}\left[\exp\left(\lambda\sup_{\zeta \in \mathbb{F}(m, w)} \left( \frac{1}{n}\sum_{i=1}^n \zeta(\tau_i) -  \frac{1}{n}\sum_{i=1}^n \zeta(\tau'_i) \right) \right) \right]\\
\end{align*}
where $\tau'_i = \{s_i^h, a^{'h}_i\}_{h=1}^H$. Using symmetrization, this is
\begin{align*}
    &\leq \E_{\Prob_\beta, \Re}\left[\exp\left(2\lambda\sup_{\zeta \in \mathbb{F}(m, w)} \frac{1}{n}\sum_{i=1}^n \xi_i \zeta(\tau_i)  \right) \right] \\ 
    &= \E_{\Prob_\beta, \Re}\left[ \sup_{\zeta \in \mathbb{F}(m, w)} \exp \left(2\lambda \frac{1}{n}\sum_{i=1}^n \xi_i \zeta(\tau_i)  \right) \right] \\ 
    &= \E_{\Prob_\beta, \Re}\left[ \sup_{\varrho, t} \exp \left(2\lambda \frac{\varrho}{nm}\sum_{i=1}^n \xi_i \sum_{h=1}^H w(\tau_i^{1:h-1}) \sum_{j=1}^m w(A_i^h, S_i^h)\one_{\{x^j_{S_i^h, A_i^h} \leq t\}} \right) \right] 
\end{align*}

Now for each $h, j$, permute the indices of $i$ with a new indexing $(i)$ such that 
$$x^j_{S_{(1)}^h, A_{(1)}^h} \leq \ldots \leq x^j_{S_{(i)}^h, A_{p(i)}^h} \leq \ldots \leq x^j_{S_{(n)}^h, A_{(n)}^h}. $$ 

Then we have 
\allowdisplaybreaks
\begin{align*}
    &\E_{\Prob_\beta, \Re}\left[ \sup_{\varrho, t} \exp \left(2\lambda \frac{\varrho}{nm}\sum_{i=1}^n \xi_i \sum_{h=1}^H w(\tau_i^{1:h-1}) \sum_{j=1}^m w(A_i^h, S_i^h)\one_{\{x^j_{S_i^h, A_i^h} \leq t\}} \right) \right] \\ 
    &= \E_{\Prob_\beta, \Re}\left[ \max_{\varrho, k} \exp \left(2\lambda \frac{\varrho}{nm}\sum_{h=1}^H \sum_{j=1}^m \sum_{i=1}^k \xi_{(i)} w(\tau_{(i)}^{1:h-1}) w(A_{(i)}^h, S_{(i)}^h)\right) \right] \\
    &\leq \E_{\Prob_\beta, \Re}\left[ \max_{\varrho, j, k} \exp \left(2\lambda \frac{\varrho}{n}\sum_{i=1}^k \xi_{(i)} \sum_{h=1}^H  w(\tau_{(i)}^{1:h-1}) w(A_{(i)}^h, S_{(i)}^h) \right) \right] \\

    &\leq 2\sum_{j=1}^m \E_{\Prob_\beta, \Re}\lb \exp\lp \frac{2\lambda}{n}\sum_{i=1}^n  \xi_{j(i)} \sum_{h=1}^H  w(\tau_{j(i)}^{1:h})\rp \one_{\{ \sum_{i=1}^n \xi_{j(i)} \sum_{h=1}^H  w(\tau_{j(i)}^{1:h})\geq 0 \}} \rb \\ 
    &\leq 2m \E_{\Prob_\beta, \Re}\lb \exp\lp \frac{2\lambda}{n}\sum_{i=1}^n  \xi_{j(i)} \sum_{h=1}^H  w(\tau_{j(i)}^{1:h})\rp \rb \\ 
    &\leq 2m \exp\lp \frac{2\lambda^2 w_{max}(w_{max}^H - 1)}{n(w_{max} - 1)} \rp 
\end{align*}

And we have the result, 
\begin{align*}
    &\Prob_\beta\lp \cdot \geq \epsilon + \frac{1}{m} \rp \leq 4m \exp\lp \frac{2\lambda^2 w_{max}(w_{max}^H - 1)}{n(w_{max} - 1)} - \lambda\epsilon \rp  \\ 
    &\leq 4m \min_{\lambda  > 0}\exp\lp \frac{2\lambda^2 w_{max}(w_{max}^H - 1)}{n(w_{max} - 1)} - \lambda\epsilon \rp \\ 
    &= 4m \exp\lp \frac{-n\epsilon (w_{max}-1)^2}{w_{max}^2 (w_{max}^H - 1)^2} \rp
\end{align*}
where the second line results because the inequality holds for all $\lambda > 0$, so we can minimize over $\lambda$. 

Then we have 
\begin{align*}
    \Prob_\beta\lp \cdot \geq \sqrt{\frac{w_{max}^2 (w_{max}^H - 1)^2}{n (w_{max} - 1)^2}\log \frac{4m}{\delta}} + \frac{1}{m} \rp \leq \delta
\end{align*}

Combining this with Lemma~\ref{lem:is_error} to bound the first term gives the result and choice of $m \propto \sqrt{n}$ gives the result.





\subsection{Cramer-Rao Variance Lower Bound of Off-Policy CDF Estimation (Proof of Theorem~\ref{thm:lower_bound})}

We prove a Cramer-Rao lower bound for UDAGs, defined below.

\begin{definition}[Unique Directed Acyclic Graph]\label{def:dag}
    An \MDP is unique directed acyclic graph (UDAG) if the state and action spaces are finite. A state only occurs at a particular horizon $h$, that is, for any $s \in \Ss$, there exists a unique $h$ such that $\max_{\pi}P(s_h=s|\pi) > 0$. In addition, each state has a unique reward function. For simplicity, assume $\gamma = 1$. 
\end{definition}

\begin{proof}
First, we transform the MDP of Definition~\ref{def:dag} into one with an augmented state space, including the history of rewards encountered.  That is, a trajectory is now defined as 
$$\tau = (s_0, g_0, a_0, ... s_H, g_H, a_H, s_{H+1}, g_{H+1}), $$
where $g_h = \sum_{k=0}^{h-1} r_h$. 
Note that $P(g_{h+1}|s_h, g_h, a_h) = P(r_h|s_h,a_h)$ and $P(s_{h+1}|s_h,g_h,a_h) = P(s_{h+1}|s_h, a_h)$.

For brevity, denote $\wt{s} = (s,g)$ to be the augmented state space. 
We now derive a pointwise Constrained Cramer-Rao Bound (CCRB) for $F(t)$ in the augmented MDP, 
\[
    KU(UIU)^{-1}U^\top K^\top
\]
where $K$ is the Jacobian of $F(t)$ and
$I$ is the Fisher information matrix. 
$U$ is a matrix whose column vectors consist of an orthonormal basis for the null space of a matrix $N$, which is a block-diagonal matrix with $N_{(\wt{s},a),(\wt{s}',a')} = \one\llbrace \wt{s},a = \wt{s}',\wt{a}' \rrbrace$. 
Further, define $\eta_{\wt{s},a,\wt{s}',g'} = P(\wt{s}'|\wt{s},a)$, 
and note that $N\eta = \mathbf{1}$. 

To compute the Fisher information matrix $I$, 
\[ 
    I = \E\lb \lp \frac{\partial\log P_\beta(\tau)}{\partial \eta} \rp \lp \frac{\partial\log P_\beta(\tau)}{\partial \eta} \rp^\top \rb
\]
where $P_\beta(\tau)$ is the probability of a trajectory $\tau$ under the behavioral policy $\beta$, 
\small{
\begin{align*}
    P_\beta(\tau) &= \mu(s_1, g_1)\beta(a_1|s_1)P(s_2, |s_1, a_1)P(g_2|s_1,a_1,g_1)\ldots P(s_H,g_H|s_{H-1},g_{H-1} a_{H-1})\beta(a_H|s_H)P(s_{H+1}, g_{H+1}|s_H, g_H,a_H) \\ 
    &= \mu(\wt{s}_1)\beta(a_1|\wt{s}_1)P(\wt{s}_2, |\wt{s}_1, a_1)P(\wt{s}_2|\wt{s}_1,a_1)\ldots P(\wt{s}_H|\wt{s}_{H-1} a_{H-1})\beta(a_H|\wt{s}_H)P(\wt{s}_{H+1}|\wt{s}_H,a_H)
\end{align*}
}%
\normalsize

Define a new indicator vector $g(\tau)$ with $g(\tau)_{\wt{s}_h, a_h, \wt{s}_{h+1}} = 1$ if $(\wt{s}_h, a_h, \wt{s}_{h+1}) \in \tau$. Then, letting $\circ$ denote an elementwise operation, we have 
\[
    \frac{\partial\log P_\beta(\tau)}{\partial \eta} = \eta^{\circ -1} \circ g(\tau)
\]
which makes
\[
    I = \E\lb [1/\eta_i\eta_j]_{ij} \circ \lp g(\tau)g(\tau)^\top\rp \rb = [1/\eta_i\eta_j]_{ij} \circ \E\lb g(\tau)g(\tau)^\top \rb
\]

As a result, the elements of $I$ are as follows, where $P_M$ refers to the marginal probability: 
\begin{itemize}
    \item diagonal: $\frac{P_M(\wt{s}_h, a_h)}{P(\wt{s}_{h+1}|\wt{s}_h, a_h)}$ 
    \item row $=(\wt{s}_h, a_h, \wt{s}_{h+1})$, column $=(\wt{s}'_h, a'_h, \wt{s}'_{h+1})$: $P_M(\wt{s}_h', a_h')P_M(\wt{s}_h, a_h|\wt{s}_{h+1}')$
    \item otherwise: 0
\end{itemize}
We know calculate the term $(U^\top IU)^{-1}$, and we use a similar strategy to \cite{jiang2016doubly, huang2020importance} to avoid taking its inverse. Note that 
\[ 
    U^\top IU = U^\top \lp N^\top X^\top + I + XN\rp U. 
\]
where $X$ is arbitrary, $N$ was defined previously. Let $D =  N^\top X^\top + I + XN$. Our goal is to define $X$ such that $D$ is a diagonal matrix with diagonal identical to the diagonal of $I$. Note that $XN$ and $N^\top X^\top $ are symmetric, so we can design the former to eliminate the upper triangle of $I$, and the latter to eliminate the lower triangle. To verify we can do this, set
\begin{itemize}
    \item row $= (\wt{s}_h, a_h, \wt{s}_{h+1})$, column $=(\wt{s}_h, a_h, 2)$ : $-P_M(\wt{s}_h, a_h)$
    \item row $= (\wt{s}_h, a_h, \wt{s}_{h+1})$, column $=(\wt{s}_h', a_h', \{1 \;\text{or}\;2\})$ : $-P_M(\wt{s}_h, a_h)P_M(\wt{s}_h', a_h', \wt{s}_{h+1})\one\llbrace h < h' \rrbrace$
    \item otherwise: 0
\end{itemize}

Then with the proper choice of $U$, 
\[
    U(U^\top I U)^{-1}U^\top  = \text{Diag}(\llbrace B(\wt{s}_h, a_h) \rrbrace_0^H)
\]
where $\text{Diag}(\llbrace \cdot \rrbrace)$ is a block-diagonal matrix consisting of matrices in the set $\llbrace \cdot \rrbrace$, and 
\[
    B(\wt{s}_h, a_h) = \frac{\text{Diag}(P(\cdot|\wt{s}_h, a_h)) - P(\cdot|\wt{s}_h, a_h)P(\cdot|\wt{s}_h, a_h)^\top }{P_M(\wt{s}_h, a_h)}
\]
where $P(\cdot|\wt{s}_h, a_h)$ is the transition vector. 

We now need to calculate the Jacobian $K$. 
Recall that the estimation objective is
\begin{align*}
    F(t) &= \sum_{\wt{s}_1}\mu(\wt{s}_1)\sum_{a_1}\pi(a_1|\wt{s}_1)\ldots \sum_{s_H}P(\wt{s}_H|\wt{s}_{H-1}, a_{H-1})\one\llbrace g_H \leq t \rrbrace \\
\end{align*}

Then $K(\wt{s}_h, a_h, \wt{s}_{h+1})$ is 
\begin{align*}
    \frac{\partial F(t)}{\partial P(r_h, s_{h+1}|s_h, a_h)} &= \frac{\partial F(t)}{\partial P(s_{h+1}, g_{h+1}|s_h,g_h, a_h)} = \Prob(s_h, g_h, a_h) P(r_h|s_h, a_h)P(s_{h+1}|s_h, a_h)F_{s_{h+1}}(t - g_h - r_h)
\end{align*}

Denote by $K_{(s_h, a_h, \cdot)}$ the vector fragment of $K$ whose index tuple starts with $s_h, a_h$. Then we have that 
\begin{align}
    KU(UIU)^{-1}U^\top K^\top &= \sum_{h=1}^{H+1}\sum_{\wt{s}_h, a_h}\lp K_{(\wt{s}_h, a_h, \cdot)} \rp ^\top B(\wt{s}_h, a_h)K_{(\wt{s}_h, a_h, \cdot)} \label{eq:ku}
\end{align}

We have 
\begin{align*}
    &(K_{(s_h, g_h, a_h, \cdot)})^\top B(s_h, g_h, a_h) K_{(s_h, g_h, a_h, \cdot)} \\
    &= \frac{\Prob(s_h, g_h, a_h)^2}{\Prob_\beta(s_h, g_h, a_h)} \Bigg( \sum_{r_h}P(r_h|s_h, a_h) \sum_{s_{h+1}}P(s_{h+1}|s_h, a_h) F_{s_{h+1}}(t - g_{h} - r_h)^2 \\
    &\quad- \lp \sum_{r_h}P(r_h|s_h, a_h) \sum_{s_{h+1}}P(s_{h+1}|s_h, a_h)F_{s_{h+1}}(t - g_{h} - r_h)\rp^2 \Bigg) \\
    &= \frac{\Prob(s_h, g_h, a_h)^2}{\Prob_\beta(s_h, g_h, a_h)}  \Var_{r_h, s_{h+1}|s_h, a_h}\lb F_{s_{h+1}}(t - g_{h+1})\rb 
\end{align*}

Under Definition~\ref{def:dag}, we have that $\Prob(g_h) = \Prob(\tau_h)$ so the lower bound is 
\begin{align*}
    \sum_{h=1}^H \E_{\Prob_\beta} \lb w_{1:h-1}^2 \Var_{r_h, s_{h+1}|s_h, a_h}\lb F_{s_{h+1}}\lp t - \sum_{k=1}^h r_k \rp \rb \rb. 
\end{align*}
\end{proof}

\paragraph{Comparison with DR Variance.} When expanded over $H$ horizons, the DR variance~\eqref{lem:dr_bias_var} in MDP of Definition~\ref{def:dag} is
\begin{align*}
    \sum_{h=1}^{H} \E_{\Prob_\beta}\lb w_{1:h-1}^2 \Var_h\lb w(A_h,S_h)\Delta_{S_h,A_h}\lp t - \sum_{k=1}^h R_k\rp  \Big| S_h \rb  \rb  + \sum_{h=1}^H \E_{\Prob_\beta}\lb w_{1:h-1}^2 \Var_h\lb F_{S_h}\lp t - \sum_{k=1}^h R_k\rp \rb  \rb
\end{align*}
where $\Delta_{s,a}(t) = \ol{F}_{s,a}(t) - F(t)$. Thus when $\ol{F} = F$, the first term of the above expression goes to 0 and the DR variance becomes 
\begin{align*}
    \sum_{h=1}^H \E_{\Prob_\beta}\lb w_{1:h-1}^2 \Var_h\lb F_{S_h}\lp t - \sum_{k=1}^h R_k\rp \rb  \rb
\end{align*}
which attains the Cramer-Rao lower bound.

\section{Proofs for Risk Estimation (Section~\ref{sec:risk})}
\subsection{CVaR Risk Lower Bound (Proof of Theorem~\ref{thm:cvar_lb})}
To prove the lower bound for CVaR estimation, we make an adaptation of the proof in~\ref{proof:mab_cdf_lb}. 
From Le Cam's method, we have 
\begin{equation}
    \inf_{\wh F } \sup_{P_1, P_2 \in \mathcal{P}} \E_P\lb \vert \wh{\rho} - \rho_P^\pi \vert \rb \geq \frac{1}{8} \vert \rho_1^\pi - \rho_2^\pi \vert e^{-n \dkl(P_1 \Vert P_2)}
\end{equation}

For a distortion risk functional $\rho$, 
\begin{align*}
    \left\vert \rho_1^\pi - \rho_2^\pi \right\vert &= \left\vert \int_0^D g(1 - F_1^\pi(t)) - g(1 - F_2^\pi(t)) dt \right\vert 
\end{align*}

In the case of $\cvara$, with $\alpha \in [0, 1)$, we have $g(x) = \min\{ \frac{x}{1-\alpha}, 1\}$ so
\begin{align*}
    \left\vert \cvara(F_1^\pi) - \cvara(F_2^\pi) \right\vert &= \left\vert \int_0^D \min\llbrace \frac{1 - F_1^\pi(t)}{1-\alpha}, 1\rrbrace - \min\llbrace \frac{1 - F_2^\pi(t)}{1-\alpha}, 1\rrbrace dt \right\vert
\end{align*}
Note that, under this definition, the expected return is CVaR with level $\alpha = 0$. 

Recall the Bernoulli problem instances we previously used to prove the lower bound for MABs. For the MAB $P_2$, we set the arm reward distribution to be $R_2(0|a) = c$ for all $a \in \A$. For the MAB $P_1$, we set the arm reward distribution to be $R_1(0|a) = c + \delta(a)$ for some $\delta(a) \in [0, 1 - c]$. Then 
\begin{align*}
     \left\vert \cvara(F_1^\pi) - \cvara(F_2^\pi) \right\vert &= \left\vert \int_0^1 \min\llbrace \frac{1 - F_1^\pi(t)}{1-\alpha}, 1\rrbrace - \min\llbrace \frac{1 - F_2^\pi(t)}{1-\alpha}, 1\rrbrace dt \right\vert \\
     &= \left\vert \min\llbrace \frac{1 - F_1^\pi(t=0)}{1-\alpha}, 1\rrbrace - \min\llbrace \frac{1 - F_2^\pi(t=0)}{1-\alpha}, 1\rrbrace \right\vert\\
     &= \left\vert \min\llbrace \frac{1 - c - \sum_a \pi(a)\delta(a)}{1-\alpha}, 1\rrbrace - \min\llbrace \frac{1 - c}{1-\alpha}, 1\rrbrace \right\vert
\end{align*}

Then note that the difference above has three possibilities: 
\[
    = 
    \begin{cases}
        0, & \text{if}\;  1-\alpha \leq 1 - c - \sum_a \pi(a)\delta(a) \\
        \frac{|\alpha - c|}{1-\alpha},  & \text{if}\; 1 - c - \sum_a \pi(a)\delta(a) \leq 1-\alpha \leq 1 - c \\
        \frac{\sum_a \pi(a)\delta(a)}{1-\alpha}, & \text{if}\; 
        1 - c \leq 1-\alpha
    \end{cases}
\]

By setting $c = \alpha$ and $\delta(a) \in [0, 1- \alpha]$, 
it can be seen that 
the first scenario (a vacuous lower bound) is impossible, 
and that the third scenario encompasses the second scenario. 
That is, if $\delta(a) = 0\; \forall a$, then both the second and third cases become 0. 

Then since 
\begin{align*}
    \dkl(P_1 || P_2) = \sum_a \beta(a) \dkl(R_1(a) || R_2(a)) \leq 4\sum_a \beta(a)\delta(a)^2, 
\end{align*}

Our optimization problem boils down to 
\begin{equation*}
    \max_{\delta \in \R^\A} \frac{\sum_a \pi(a)\delta(a)}{1-\alpha} \quad\text{s.t.}\quad \sum_a \beta(a)\delta(a)^2 \leq \frac{1}{8n},\; 0\leq \delta \leq 1-\alpha
\end{equation*}

Thus we have the Lagrangian 
\begin{align*}
    L(\delta, \lambda, \nu) &= - \frac{\sum_a \pi(a)\delta(a)}{1-\alpha} + \lambda\lp \sum_a \beta(a)\delta(a)^2 - \frac{1}{8n} \rp + \sum_a \nu(a)\lp \delta(a) - 1 + \alpha\rp \\ 
    &= -\frac{\lambda}{8n} + \sum_a \llbrace \lambda \beta(a) \delta^2(a)  + \delta(a)\lp\nu(a) - \frac{\pi(a)}{1-\alpha}\rp - (1-\alpha)\nu(a) \rrbrace 
\end{align*}

The optimal solution is 
\begin{align*}
    \delta^*(a) = \frac{\frac{\pi(a)}{1-\alpha} - \nu(a)}{2\lambda\beta(a)}
\end{align*}

When we plug this in, we have 
\begin{align}
    - \frac{ \sum_a \pi(a) \delta^*(a)}{1-\alpha}  &= \max_{\lambda \geq 0, \nu \geq 0} g(\lambda, \nu) = \max_{ \nu \geq 0} \llbrace -\sqrt{\frac{1}{8n} \sum_a \frac{\lp\frac{\pi(a)}{1-\alpha} - \nu(a)\rp^2}{\beta(a)}} - (1-\alpha)\sum_a \nu(a) \rrbrace \nonumber \\ 
    &= -\min_{\nu \geq 0} \llbrace \sqrt{\frac{1}{8n} \sum_a \frac{\lp\frac{\pi(a)}{1-\alpha} - \nu(a)\rp^2}{\beta(a)}} + (1-\alpha)\sum_a \nu(a) \rrbrace \label{cvar_lb_optimization}
\end{align}

Finally, we have from Lemma~\ref{lem:cvar_opt_asymp} that 
\begin{align*}
    \min_{\nu \geq 0} \llbrace \sqrt{\frac{1}{8n} \sum_a \frac{\lp\frac{\pi(a)}{1-\alpha} - \nu(a)\rp^2}{\beta(a)}} + (1-\alpha)\sum_a \nu(a) \rrbrace \asymp \frac{1}{1-\alpha} \llbrace \sum_{a \in S^*}\pi(a) + \sqrt{\frac{\sum_{a \notin S} \beta(a) w(a)^2}{n}}\rrbrace
\end{align*}

\begin{lemma}\label{lem:cvar_opt_asymp}
    \begin{align}
        \min_{\nu \geq 0} \llbrace \sqrt{\frac{1}{8n} \sum_a \frac{\lp\frac{\pi(a)}{1-\alpha} - \nu(a)\rp^2}{\beta(a)}} + (1-\alpha)\sum_a \nu(a) \rrbrace \asymp \sum_{a \in S^*}\pi(a) + \frac{1}{1-\alpha}\sqrt{\frac{\sum_{a \notin S} \beta(a) w(a)^2}{n}}\label{eq:risk_asymp}
    \end{align}
\end{lemma}

\begin{proof}
    First, note that $0 \leq \nu^* \leq \pi(a)$ from the optimization problem~\eqref{cvar_lb_optimization}. 
    Further, WLOG we can assume that $\pi(a) > 0$ and $\beta(a) > 0$. This is because if $\pi(a) = 0$, then $w(a) = 0$, which contributes 0 to both sides of the equation. Similarly, if $\beta(a) = 0$, then $\nu^*(a) = \pi(a)$, which means $a \in S^*$, and $\pi(a)$ is contributed to both sides of the equation. 
    
    We separate the proof into three cases: $v^* = \frac{\pi}{1-\alpha}$, $v^* = 0$, $0 \neq v^* \neq \frac{\pi}{1-\alpha}$. 
    
    \paragraph{$\mathbf{v^* = \pi}$.} In this case, $S^* = [k]$, and the LHS of~\eqref{eq:risk_asymp} is 
    \begin{align*}
        (1-\alpha)\sum_a \frac{\pi(a)}{1-\alpha} = 1
    \end{align*}
    while the RHS is 
    \begin{align*}
        \sum_a \pi(a) = 1. 
    \end{align*}
    
    \paragraph{$\mathbf{v^* = 0}$.} In this case, $S^* = \{\}$, and the LHS of~\eqref{eq:risk_asymp} is 
    \begin{align*}
         \sqrt{\frac{1}{8n} \sum_a \frac{\frac{\pi(a)}{1-\alpha}^2}{\beta(a)}} &= \frac{1}{1-\alpha}\sqrt{\frac{1}{8n} \sum_{a}\beta(a)w(a)^2},
    \end{align*}
    while the RHS is 
    \begin{align*}
        \frac{1}{1-\alpha}\sqrt{\frac{\sum_{a} \beta(a) w(a)^2}{n}}. 
    \end{align*}
    
    \paragraph{$\mathbf{0 \neq v^* \neq \pi}.$} First, we write the optimality conditions of~\eqref{cvar_lb_optimization}: 
    \begin{align}
        \lp \frac{w(a)}{1-\alpha} - \frac{\nu^*(a)}{\beta(a)}\rp^2 = 8n(1-\alpha)^2 \sum_a \frac{\lp \frac{\pi(a)}{1-\alpha} - \nu^*(a)\rp^2}{\beta(a)} \quad\quad&\text{for}\; a \in S^* \label{eq:optim_1} \\ 
        w(a)^2 \leq 8n(1-\alpha)^4 \sum_a \frac{\lp \frac{\pi(a)}{1-\alpha} - \nu^*(a)\rp^2}{\beta(a)} \quad\quad&\text{for}\; a \notin S^* \label{eq:optim_2}
    \end{align}

    Denote 
    \begin{align*}
        T_1 &:= \sum_{a \in S^*}  \frac{\lp \frac{\pi(a)}{1-\alpha} - \nu^*(a)\rp^2}{\beta(a)}\\
        T_2 &:= \sum_{a \notin S^*}  \frac{\lp \frac{\pi(a)}{1-\alpha} - \nu^*(a)\rp^2}{\beta(a)} = \frac{1}{(1-\alpha)^2} \sum_{a \notin S^*}\beta(a)w(a)^2
    \end{align*}
    
    Next, we will show that $T_1 = \frac{1-\epsilon}{\epsilon} T_2$ for some $\epsilon \in (0,1)$ such that $\beta(S^*) = (1-\epsilon)/8(1-\alpha)^2n$. 
    
    Summing~\eqref{eq:optim_1} over $a \in S^*$, we have 
    \begin{align*}
        T_1 &= \sum_{a \in S^*}\beta(a) \lp \frac{w(a)}{1-\alpha} - \frac{\nu^*(a)}{\beta(a)}\rp^2 \\ 
        &= 8n(1-\alpha)^2\beta(S^*) \sum_a \frac{\lp \frac{\pi(a)}{1-\alpha} - \nu^*(a)\rp^2}{\beta(a)} \\ 
        &= 8n(1-\alpha)^2\beta(S^*)(T_1 + T_2)
    \end{align*}
    
    which implies that $S^* \neq [k]$. Then since $w(a) > 0$ and $\nu^* \neq 0$, both $T_2 > 0$ and $\beta(S^*) > 0$. Then we have 
    \begin{align*}
        T_1 > 8n(1-\alpha)^2 \beta(S^*) T_1
    \end{align*}
    which implies that 
    $$ \frac{1}{8n(1-\alpha)^2} > \beta(S^*) > 0.$$
    
    This means that for some $\epsilon \in (0,1)$, 
    \begin{align*}
        \beta(S^*) = \frac{1-\epsilon}{8n(1-\alpha)^2}. 
    \end{align*}
    
    Then combining this with the previous relationship, 
    \begin{align*}
        T_1 &= 8n(1-\alpha)^2 \frac{1-\epsilon}{8n(1-\alpha)^2} (T_1 + T_2) \\ 
        &= \frac{1-\epsilon}{\epsilon}T_2. 
    \end{align*}
    
    Next, we show that~\eqref{cvar_lb_optimization} has optimal value 
    \begin{align*}
        \pi(S^*) + \sqrt{\frac{T_2}{8n}}\epsilon. 
    \end{align*}
    
    From~\eqref{eq:optim_1}, for $a \in S^*$ we have that
    \begin{align*}
        \nu^*(a) &= \frac{\pi(a)}{1-\alpha} - \beta(a)\sqrt{8(1-\alpha)^2n(T_1 + T_2)}. 
    \end{align*}
    
    Then
    \begin{align*}
        \sqrt{\frac{1}{8n} \sum_a \frac{\lp\frac{\pi(a)}{1-\alpha} - \nu^*(a)\rp^2}{\beta(a)}} + (1-\alpha)\sum_a |\nu^*(a)| =& \sqrt{\frac{1}{8n}(T_1 + T_2)} \\
        &+ (1-\alpha)\sum_{a \in S^*}\lp \frac{\pi(a)}{1-\alpha} - \beta(a)\sqrt{8(1-\alpha)^2n(T_1 + T_2)} \rp \\
        =& \pi(S^*) + \lp 1 -  \beta(S^*)8n(1-\alpha)^2\rp \sqrt{\frac{1}{8n}(T_1 + T_2)} \\ 
        =& \pi(S^*) + \epsilon\sqrt{\frac{1}{8n}(T_1 + T_2)} \\ 
        =& \pi(S^*) + \sqrt{\frac{\epsilon}{8n}T_2}
    \end{align*}
    
    Going back to our original problem, note that for $\epsilon \in [1/2, 1]$ we achieve the desired equivalence. Thus we focus on the case where $\epsilon \in (0, 1/2)$, which means that $\frac{1}{16(1-\alpha)^2n} < \beta(S^*) < \frac{1}{8(1-\alpha)^2n}$. Assume WLOG that the actions are ordered according to likelihood ratios $w(1) \leq \ldots \leq w(k)$. From~\eqref{eq:optim_1} and~\eqref{eq:optim_2}, $S^* = \{t+1, \ldots, k\}$ for some $t \in [k]$. Then applying~\eqref{eq:optim_1}, we have that for action $a=t+1$, 
    \begin{align*}
        \frac{w(t+1)^2}{(1-\alpha)^2} \geq \lp \frac{w(t+1)}{1-\alpha} - \frac{\nu^*(t+1)}{\beta(t+1)}\rp^2 = 8n(1-\alpha)^2 (T_1 + T_2) = \frac{8n(1-\alpha)^2T_2}{\epsilon}
    \end{align*}
    
    Using the fact that $\frac{\pi(S^*)}{\beta(S^*)} \geq w(t+1)$, we have that 
    \begin{align*}
        \frac{\pi(S^*)}{\beta(S^*)} \geq (1-\alpha)^2 \sqrt{\frac{8nT_2}{\epsilon}}
    \end{align*}
    
    Then since $\beta(S^*) > \frac{1}{16(1-\alpha)^2n}$, 
    \begin{align*}
        \pi(S^*) \geq  \sqrt{\frac{T_2}{8n\epsilon}} \geq \sqrt{\frac{T_2}{4n}\epsilon}
    \end{align*}
    
    Then 
    \begin{align*}
        \sqrt{\frac{1}{8n} \sum_a \frac{\lp\frac{\pi(a)}{1-\alpha} - \nu^*(a)\rp^2}{\beta(a)}} + (1-\alpha)\sum_a |\nu^*(a)| &= \pi(S^*) + \sqrt{\frac{\epsilon}{8n}T_2} \asymp \pi(S^*) \asymp \pi(S^*) + \sqrt{\frac{T_2}{n}}
    \end{align*}
    which achieves the target equivalence using the definition of $T_2$. 
    
    Finally, examining the optimality condition for $a \notin S^*$~\eqref{eq:optim_2} shows that $\nu^* = 0$ when 
    \begin{align*}
        n \geq \frac{w_{max}^2}{8(1-\alpha)^2 \sum_a \beta(a)w^2(a)},
    \end{align*}
    
    which means that the lower bound is 
    \begin{align*}
        \frac{cD}{1-\alpha}\sqrt{\frac{\sum_a \beta(a)w(a)^2}{n}}
    \end{align*}
    for some universal constant $c$, giving the theorem statement. 
\end{proof}

\newpage

\section{Experiments}\label{appendix:experiments}

\subsection{Implementation}

We describe each environment and their modeling methods below. For both environments we use a maximum horizon $H = 200$ and $\gamma = 1$. 

\paragraph{Cliffwalk.} Cliffwalk~\cite{sutton2018reinforcement} is a $4 \times 12$ tabular environment where the agent must travel from the start state (lower left corner) to the goal state (lower right corner). The agent can take one of four cardinal actions (left, down up, right). The bottom row of cells represents a cliff, and moving into the cliff incurs a cost of 100. Each other state incurs a cost of 1, thus incentivizing the agent to take the shortest path between start a goal. As the original Cliffwalk is deterministic, we introduce a random transition towards the cliff with probability $p = 0.05$ in each state. 

The target policy $\pi$ is learned via Q-learning, and the behavioral policy is a mixture between the target and a uniform policy, that is for a constant $\lambda \in [0, 1]$, $\beta = \lambda\pi + (1-\lambda)\text{UNIF}$. 
For estimators (\DM, \DR, \WDR) involving models $\ol{F}$, half of the dataset is used to construct the model and the other half is used for estimation. This process is then repeated with the halves switched, and the resulting estimate is an average of the two. First, an estimate of the MDP $\ol{M}$ is constructed using empirical averages for the transitions and rewards. Following~\cite{thomas2015high}, we assume that the model is given imperfect information of the horizon as $H=201$. 
$\ol{F}$ is then computed recursively via the relations from Section~\ref{sec:dm_cdf}: 
\begin{align*}
    &\ol{F}_{s_h, a_h}^{H+1-h}(t) = \E_{\overline{P}, \overline{R}}\lb \ol{F}_{S_{h+1}}^{H-h}\lp\frac{t- R_h}{\gamma}\rp \rb \\ 
    &\ol{F}_{s_h}^{H+1-h}(t) = \E_\pi\lb \ol{F}_{s_h, A_h}^{H+1-h}(t)\rb
\end{align*}
The Cliffwalk results shown in Figure~\ref{fig:error} were averaged over 1000 repetitions. 

\paragraph{Simglucose.} In Simglucose, the agent must control insulin bolus injections to a patient with type 1 diabetes. The state is a continuous vector with the patient's blood glucose levels and the carbohydrate intake from the last meal. We discretize the space of possible bolus injections into 6 actions. The agent receives a reward according to whether the patient's blood glucose levels are within acceptable limits or not. If the patient's blood glucose exceeds 180, a condition called hyperglycemia, the agent receives a reward of -1. If it falls under 70, called hypoglycemia, the agent receives a reward of -2. Otherwise, the agent receives a reward of +1. 

The target policy is the built-in controller for the patient, and the behavioral policy is defined as the same way for Cliffwalk. As the state space is continuous, the model was built by first discretizing the state space, then using empirical averages to estimate the MDP $\ol{M}$ in the discretized space, upon which $\ol{F}$ is calculated. 

The Simglucose results shown in Figure~\ref{fig:error} were averaged over 100 repetitions.

\subsection{Mean vs Plug-in + CDF}
As an additional experiment, we compare the plug-in mean estimate on the estimated CDFs with direct mean estimation for existing $\IS, \WIS$ and $\DR$~\cite{jiang2016doubly} estimates in the Cliffwalk environment.  For \WIS, at low sample sizes the direct mean estimates can outperform the CDF estimates, but reach the same error as $n$ increases. For the $\IS$ and $\DR$ estimates, however, the CDF version of the estimates perform just as well, if not better, than direct mean estimation. 

\begin{figure}[H]
    \centering
    \includegraphics[width=\textwidth]{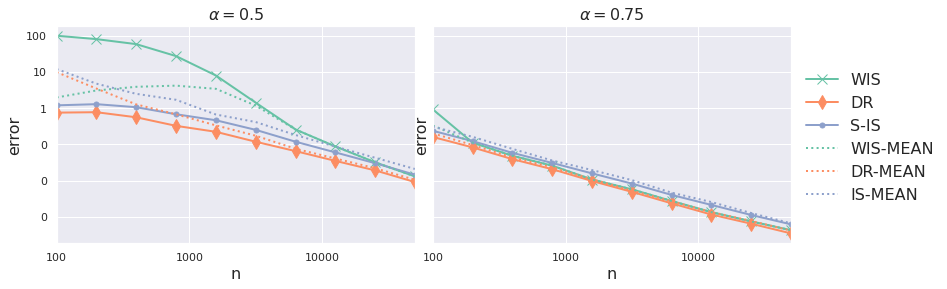}
    \caption{Normalized MSE for different $\lambda$ in the Cliffwalk environment. Dashed lines show mean analogues of CDF estimates. }
    \label{fig:mean}
\end{figure}

\end{document}